\documentclass[twoside]{article}

\usepackage[accepted]{_AISTATS/aistats2023}

\usepackage[utf8]{inputenc} %
\usepackage[T1]{fontenc}    %
\usepackage[hidelinks]{hyperref}       %
\usepackage{url}            %
\usepackage{booktabs}       %
\usepackage{amsfonts}       %
\usepackage{nicefrac}       %
\usepackage{microtype}      %

\usepackage{graphicx}

\usepackage{amssymb, mathtools, amsmath}

\usepackage{amssymb, mathtools, amsmath}
\usepackage{tcolorbox}
\usepackage{xcolor}

\usepackage{microtype}
\usepackage{graphicx}
\usepackage{booktabs} %
\usepackage[small, labelfont=bf]{caption}
\usepackage[inline]{enumitem}

\usepackage[colorinlistoftodos]{todonotes}
\usepackage{enumerate}

\usepackage{wrapfig}
\usepackage{enumitem}
\usepackage{multicol}

\usepackage{subfig}

\usepackage{pifont}

\usepackage{algorithm}

\hypersetup{
    colorlinks = true,
    citecolor=blue,
    linkcolor=blue,
    urlcolor=blue
}
\usepackage{amsthm}
\definecolor{darkgreen}{rgb}{0.0, 0.5, 0.0}
\definecolor{blue}{rgb}{0.0, 0.47, 0.75}
\definecolor{dartmouthgreen}{rgb}{0.05, 0.5, 0.06}
\definecolor{drab}{rgb}{0.59, 0.44, 0.09}
\definecolor{navyblue}{rgb}{0.0, 0.0, 0.5}

\definecolor{darkgreen}{rgb}{0.0, 0.5, 0.0}
\definecolor{blue}{rgb}{0.0, 0.47, 0.75}
\definecolor{dartmouthgreen}{rgb}{0.05, 0.5, 0.06}
\definecolor{drab}{rgb}{0.59, 0.44, 0.09}
\definecolor{navyblue}{rgb}{0.0, 0.0, 0.5}

\newcommand{\vc}{\mathbf{c}}

\newcommand{\vy}{\mathbf{y}}

\newcommand{\cmark}{\ding{51}}%
\newcommand{\xmark}{\ding{55}}%

\usepackage[compact]{titlesec}
\titlespacing{\section}{0pt}{2ex}{1ex}

\titlespacing{\subsection}{0pt}{1ex}{0ex}
\titlespacing{\subsubsection}{0pt}{0.5ex}{0ex}

\newtheorem{theorem}{Theorem}
\newtheorem{corollary}{Corollary}[theorem]

\usepackage[round]{natbib}

\begin{document}

\twocolumn[

\aistatstitle{Just Mix Once: 
Worst-group Generalization by Group Interpolation}

\aistatsauthor{ Giorgio Giannone\footnotemark[1]
\And Serhii Havrylov 
\And  Jordan Massiah 
\And Emine Yilmaz 
\And Yunlong Jiao\footnotemark[1]}

\aistatsaddress{ 
DTU\footnotemark[2]
\And Amazon 
\And  Amazon 
\And Amazon, UCL 
\And Amazon}
]

\footnotetext[1]{Correspondence to:
\\ 
\texttt{gigi@dtu.dk}, \texttt{jyunlong@amazon.co.uk}.
}

\footnotetext[2]{
Work done when the first author was at Amazon Cambridge.}

\begin{abstract}
    Advances in deep learning theory have revealed how average generalization relies on superficial patterns in data.
	The consequences are brittle models with poor performance with shift in group distribution at test time.
	When group annotation is available, we can use robust optimization tools to tackle the problem. 
	However, identification and annotation are time-consuming, especially on large datasets.
	A recent line of work leverages self-supervision and oversampling to improve generalization on minority groups without group annotation.
	We propose to unify and generalize these approaches using a class-conditional variant of mixup tailored for worst-group generalization. 
	Our approach, Just Mix Once (JM1), interpolates samples during learning, augmenting the training distribution with a continuous mixture of groups.
	JM1 is domain agnostic and computationally efficient, can be used with any level of group annotation, and performs on par or better than the state-of-the-art on worst-group generalization. 
	Additionally, we provide a simple explanation of why JM1 works.
\end{abstract}

\section{Introduction}

Supervised learning aims to fit a model on a train set to maximize a global metric at test time~\citep{vapnik1999nature, james2013introduction, hastie2009elements}.
However, the optimization process can exploit spurious correlations between the target $y$ and superficial patterns $c$ in the data~\citep{geirhos2018imagenet, geirhos2018generalisation, hermann2020origins, recht2019imagenet}.
In this work, we are especially interested in the setting in which labels and patterns can be categorized into groups $g=(c, y)$ and we study generalization performance in the presence of shifts in group distribution at test time.

For example, in the case of \texttt{coloredMNIST} (Figure~\ref{figure:identification}), a model can use information about \texttt{color} (spurious correlation with the target class) instead of \texttt{shape} to classify a digit (almost all \texttt{6}s are blue at train time) during training. 
Then, in the presence of a shift in group distribution at test time~\citep{hendrycks2021many}, the model relies on the color bias with degradation in generalization performance (for example, \texttt{6}s are green and blue with equal proportion at test time).
Powerful deep learning models exploit easy superficial correlations~\citep{geirhos2018imagenet, geirhos2017comparing}, such as texture, color, background~\citep{hermann2020origins, hendrycks2019benchmarking} to improve average generalization. 
Solving this issue involves constraining the model~\citep{ilyas2019adversarial, tsipras2018robustness, schmidt2018adversarially}, for example leveraging explicit constraints~\citep{zemel2013learning, hardt2016equality, zafar2017fairness}, invariances~\citep{arjovsky2019invariant, creager2021environment}, and causal structure~\citep{oneto2020fairness}. 

Similarly, in the presence of minority groups in a dataset, the model will tend to ignore such groups, relying on frequent patterns in majority groups.
Recent work has shown how to tackle unbalanced subpopulation or groups in data using explicit supervision~\citep{sagawa2019distributionally, yao2022improving}, self-supervision~\citep{nam2020learning}, and oversampling~\citep{liu2021just}, improving generalization on minority groups.
However, these methods have limitations in terms of applicability: such methods can be applied only to specific settings, like full group annotation or absence of group annotation, but fail to handle hybrid scenarios, such as partial labeling, group clustering, and changes in the train scheme.

In this work, we focus on improving worst-group generalization in realistic scenarios with any level of group annotation, devising a group annotation-agnostic method. 
Our approach, Just Mix Once (JM1), builds on supervised~\citep{sagawa2019distributionally} and self-supervised approaches~\citep{liu2021just} to improve worst-group performance, generalizing and unifying such methods. 
JM1 leverages an augmentation scheme based on a class-conditional variant of mixup~\citep{zhang2017mixup, verma2019manifold, carratino2020mixup} to interpolate groups and improve worst-group generalization. 
JM1 is scalable, does not require undersampling or oversampling, and is flexible, easily adaptable to a variety of problem configurations (full, partial, absence of group annotation) with minimal modifications.

\paragraph{Contribution.}
Our contributions are the following:

\begin{itemize}
	\item We propose a simple, general mechanism, JM1, based on a class-conditional variant of mixup, to improve generalization on minority groups in the data.
	JM1 is group annotation agnostic, i.e. it can be employed with different levels of annotations, generalizing and unifying proposed methods to improve group generalization. 
	\item We propose a novel interpretation of why interpolating majority and minority groups is effective in improving worst-group generalization, and justify it theoretically in an ideal case and empirically in realistic scenarios. 
	\item We perform extensive experiments with different levels of group annotation and ablation for the mixup and class-conditional strategy.
	JM1 is successfully employed with 
	I) fine-grained annotation, II) coarse annotation, and III) only a small validation set annotated, demonstrating that our method outperforms or is on par with the SOTA on vision and language datasets for worst-group generalization.
\end{itemize}

\begin{figure}[thbp]
	\centering
	\includegraphics[width=.45\textwidth]{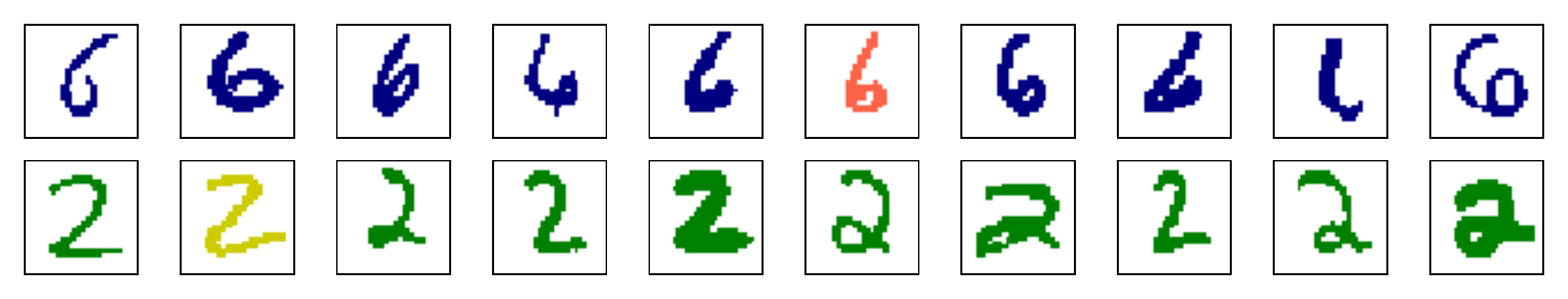}
	\quad
	\includegraphics[width=.45\textwidth]{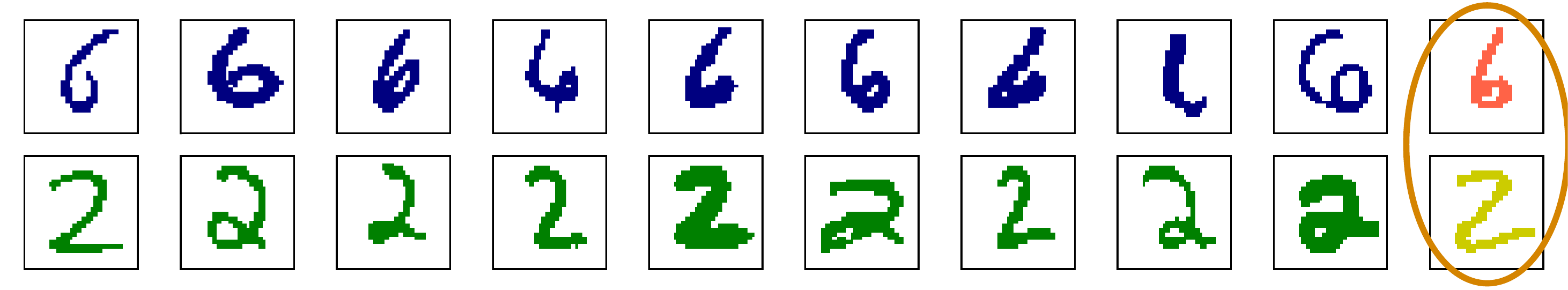}
	\caption{\textbf{Groups Identification Phase.} 
		Left: train set. 
		Right: train set partitioned.
		In the first phase, the training dynamics are exploited to split the data into two partitions.
		The assumption is that, among the samples identified as difficult (samples in the orange circle), there are typically samples from minority groups.
		The misclassification rate in the early stage of training is used as a clustering signal to estimate the group distribution: 
		patterns that are superficially frequent in the data (e.g., color, texture, background) are easy to classify and have a small loss, partitioning a \textit{majority} group;
		infrequent patterns (e.g., shape) are challenging to model and are misclassified during the early stage of training, partitioning a \textit{minority} group.
	}
\label{figure:identification}
\end{figure}

\begin{figure}[thbp]
	\centering
	\includegraphics[width=.45\textwidth]{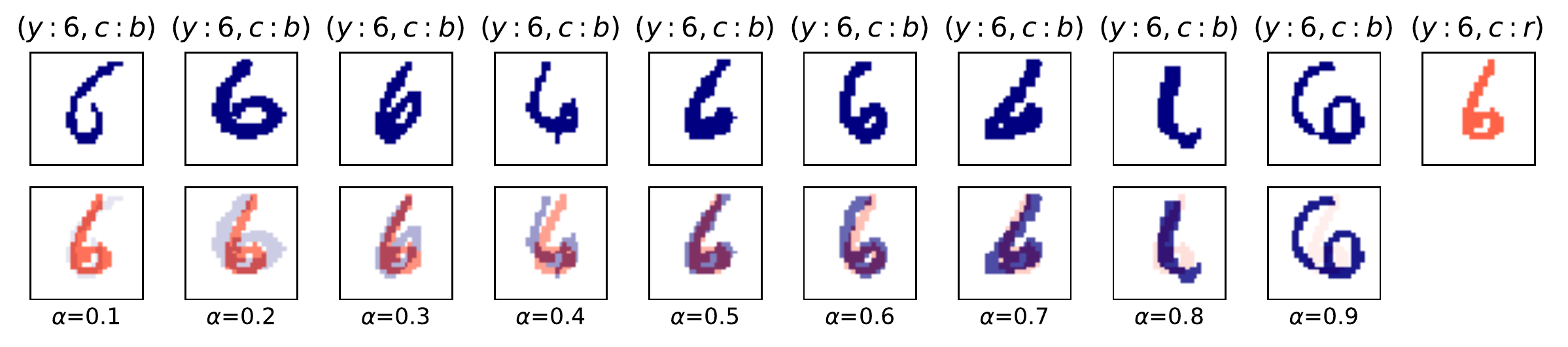}
	\quad
	\includegraphics[width=.45\textwidth]{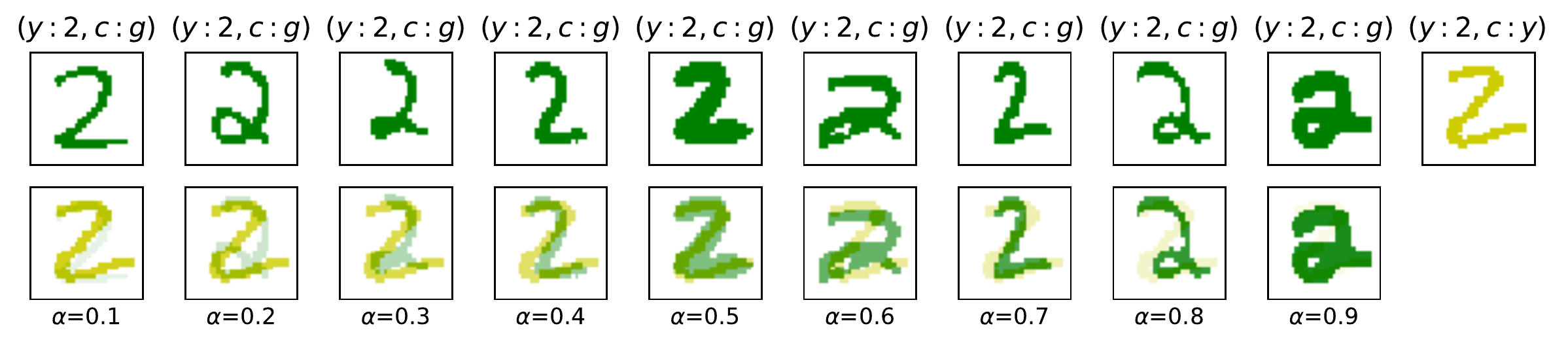}
	\caption{\textbf{Class-conditional Mixing Phase.} 
		After the identification phase, we exploit label information to augment the training data, using a MixUp-inspired strategy to augment samples from different partitions.
		We use the label information to select two samples from different partitions, i.e., $(h^i_g, y^i_g)$ and $(h^j_{\bar g}, y^j_{\bar g})$ where $y^i_g = y^j_{\bar g} = y, \, g \neq {\bar g}$.
		Then we create a mixed-up sample $(h_{mix}, y_{mix})$ s.t. $h_{mix} = \alpha~h^i_g + (1 - \alpha)~h^{j}_{\bar g},\, y_{mix} = y$.
		Note that $h$ can either be an input~\citep{zhang2017mixup} or a learned representation~\citep{verma2019manifold}.
		This simple mechanism gives us a principled and domain-agnostic way to augment samples marginalizing the group information in the data: by sampling $\alpha$, we augment the training data and build a continuous distribution of groups in the data.
		Therefore, the model cannot rely on frequent patterns because each sample has a "slightly different" group pattern.
	}
\label{figure:mixing}
\end{figure}

\section{Background}
Our work tackles the problem of improving worst-group generalization with different levels of group annotation available during training.
Here we present two main classes of methods to deal with groups in the data with and without group annotation on the train set.

Assuming that the training data $D=\{x^i, y^i\}^n_{i=1}$ are partitioned into (non-overlapping) groups $\mathcal{G}_1, \dots, \mathcal{G}_m$, through sub-populations represented by tuples $(y, c)$ of label $y$ and confounding factors $c$, we can define the following per-group and group-average loss:
\begin{align}
\begin{split}
\label{eq:group_loss}
    J(\theta; \mathcal{G}) :&= \dfrac{1}{|\mathcal{G}|} \sum_{(x_g, y_g) \in \mathcal{G}} l(x_g, y_g;\theta) \\
    J(\theta) :&= \dfrac{1}{m} \sum_{k=1}^m J(\theta; \mathcal{G}_k) .
\end{split}
\end{align}
Notice that in case $m=n$, each group contains a single data point, and the group-average loss collapses to the standard ERM loss $J_{\mathtt{ERM}}$.
Note that if we define a uniform per-group loss weight $p_k = 1/m, k=1,\dots,m$, we can rewrite the group-average loss as:
\begin{equation}
\label{eq:group_loss_proba}
J(\theta) = \sum_{k} p_k J(\theta; \mathcal{G}_k) = \mathbb{E}_{p(\mathcal{G})} J(\theta; \mathcal{G}) .
\end{equation}

Distributional Robust Optimization (DRO)~\citep{sagawa2019distributionally} aims to optimize:
\begin{equation}
J_{\mathtt{DRO}} = \max_{k} J(\theta, \mathcal{G}_k) ,
\label{eq:dro_loss}
\end{equation}
which corresponds to a \textit{pointwise} Dirac distribution $p_k = 1$ if $k = \arg\max_{k} J(\theta; \mathcal{G}_k)$ else 0 in Eq.~\ref{eq:group_loss_proba}.
This approach is different from the standard training routine, where the goal is optimizing the average error among groups.
If group information on the train set is absent, methods such as Just Train Twice (JTT)~\citep{liu2021just, nam2020learning} can be used.
JTT is a powerful approach that solves worst-group generalization in a simple two-stage approach: in the first phase, the goal is to \emph{partition} the data into two clusters: one with majority groups (frequent superficial patterns); and one with minority groups (uncommon patterns).
The assumption is that samples from minority patterns are difficult to model and frequently misclassified in the early-stage of training.
JTT uses the misclassified samples in the early stage of training as a signal to partition the data. 
Specifically, a supervised learner, parameterized by $\theta$, is trained on the data $\mathcal{D} = \{(x^i,y^i)\}_{i=1}^n$ using ERM loss (up to constant)
$J_{\texttt{ERM}}(\theta) := \sum_{(x,y) \in \mathcal{D}} l(x, y;\theta)$
with early stopping.
To avoid overfitting, a small validation set with group annotation is used to select the best identification epoch $T$ to create the appropriate partitions.
Then misclassified samples~\citep{nam2020learning} are saved in a buffer $\mathcal{B} := \{(x_b,y_b) \textrm{ s.t. } {\hat f_{T}}(x_b) \neq y_b\}$.

Once the partitions are identified, the same learner is trained for the second time on a \emph{re-weighted} version of the data, where samples in $\mathcal{B}$ are oversampled $\lambda \gg 1$ times.
The intuition is that, if the examples in the error set $\mathcal{B}$ come from challenging groups, such as those where the spurious correlation does not hold, then up-weighting them will lead to better worst-group performance~\citep{liu2021just}.
Specifically, the loss for phase II) can be written as (up to constant): 
$$
J_{\mathtt{JTT}}(\theta) := \lambda \sum_{(x_b,y_b) \in \mathcal{B}} l(x_b, y_b; \theta) + \sum_{(x_{\bar b},y_{\bar b}) \notin \mathcal{B}} l(x_{\bar b}, y_{\bar b}; \theta),
$$
where $\lambda$ is the up-sampling rate for samples in $\mathcal{B}$ identified from phase I).
A similar approach was proposed in~\citep{nam2020learning} with the difference that the groups are re-weighted after each epoch using the loss magnitude in phase one. This approach is more general, but in practice more difficult to train and scale.

\textbf{Limitations.} 
GroupDRO and JTT are effective methods but can be used only in specif scenarios: full-group annotation on the train set for the former, and group annotation on the validation set for the latter. 
These methods are challenging to adapt for hybrid scenarios where group annotation can be fine-grained for some classes (group annotation on each sample) and coarse for others (partition-based, cluster-based annotation). Such "sparse" and incomplete group annotation is the most common in practice for large, diverse datasets.
To deal with such issues, we propose a simple method to improve worst-group accuracy with different levels of annotation granularity available.
\section{Just Mix Once}
\label{sec:jm1}
We address the limitations of GroupDRO and JTT by proposing Just Mix Once (JM1).
Our goal is to improve the classification accuracy of minority groups, with or without explicit annotation on the groups at training time. In the following we will use mixing and interpolation interchangeably.
JM1 consists of two phases: phase I) discovers and clusters minority groups in the data, and phase II) uses class-conditional mixing to improve worst-group performance.
JM1 can be used with and without group annotation on the train set. 

\paragraph{I) (Optional) Groups Identification.}
In absence of group annotation on the train set, similarly to phase I) of JTT, we assume that we have annotated groups on a small validation set and resort to using a self-supervised signal~\citep{nam2020learning, arazo2019unsupervised} based on the misclassification rate (and the loss magnitude) in the early stage of training similarly to JTT.
Fig.~\ref{figure:identification} illustrates how relevant samples from minority groups are identified.
Note that self-supervised training of group identification is used only in absence of group annotation on the train set:
if group annotations are available on the train set, we always use such oracle groups to identify the majority and minority groups.

\paragraph{II) Class-conditional Interpolation.}
To generalize on minority groups at test time, a model should not rely on spurious correlations during training (e.g., texture or color) but on the signal of interest (e.g., shape).
Unlike the oversampling approach taken by JTT, we resort to a better augmentation (Fig.~\ref{figure:mixing}) to improve generalization for low-density groups.
Our mixing strategy is inspired by MixUp~\citep{verma2019manifold, zhang2017mixup, carratino2020mixup}, and we propose a novel class-conditional variant that achieves robust worst-group generalization.
Specifically, for any two samples $(h^i_g, h^j_{\bar g})$ in the input (or representation) space from the two partitions $g, {\bar g}$ (i.e., difficult/misclassified/minority group $g$ vs the other majority group ${\bar g}$) with the same class label $y^i_g = y^j_{\bar g} = y$, we mix them up using a convex interpolation:
\begin{equation}\label{eq:class_cond_mixing}
h_{mix} = \alpha~h^i_g + (1 - \alpha)~h^{j}_{\bar g} , \quad y_{mix} = y ,
\end{equation}
where $\alpha$ is the mixing parameter (details on how to choose $\alpha$ are deferred to Sec.~\ref{sec:experiments}).
Notably, we empirically show that naively implementing the standard MixUp without the proposed two-stage approach fails to generalize in terms of worst-group performance (Fig.~\ref{bar:oracle}).

Finally, JM1 uses a simple ERM loss over mixed data:
$$
J^{\mathtt{ERM}}_{\mathtt{JM1}} = \sum\nolimits_{h^i_g, g \in \mathcal{G}} \sum\nolimits_{h^{j}_{\bar g}, {\bar g}\in \bar{\mathcal{G}}} l(h_{mix}, y).   
$$

\subsection{How JM1 Works}
\label{sec:why-jm1-works}
In this subsection we explain why JM1, based on augmenting the group distribution by interpolating samples, works.

First, we discuss how our mixing strategy enables JM1 to build a "continuous" spectrum of groups in the data by sampling $\alpha$.
Using MixUp with $\alpha$ mixing rate, we generate new samples drawn from a continuous mixture (Eq.~\ref{eq:class_cond_mixing}), which has as limiting case the training distribution.
In this view, we can interpret each mixed-up sample as drawn from a "mixed-up group" $\mathcal{G}_{\alpha}$ parametrized by $\alpha$.
This means, assuming an oracle partitioning in phase I), JM1 mixes a majority and a minority group in the data at each iteration while generating a \textit{continuous} group distribution.
We denote by $J(\theta; \mathcal{G}_{\alpha})$ our per-group loss for an interpolated group $\mathcal{G}_{\alpha}$ with the mixing rate $\alpha$, and write the group-average loss marginalizing the group distribution:
\begin{align}
\begin{split}
\label{eq:group_loss_jm1}
	J_{\mathtt{mix}}(\theta) 
	: &= \mathbb{E}_{\alpha} \mathbb{E}_{p(\mathcal{G}; \alpha)} [J(\theta; \mathcal{G})] \\
	  &= \int_{\alpha} \int_{\mathcal{G}_{\alpha}} p(\alpha) p(\mathcal{G}; \alpha) J(\theta; \mathcal{G})~d\alpha~d\mathcal{G} .
\end{split}
\end{align}
Note that first sampling $\alpha$ and then applying MixUp with rate $\alpha$ (Eq.~\ref{eq:class_cond_mixing}) is equivalent to drawing samples directly from a mixed-up group $\mathcal{G}_{\alpha}$ drawn implicitly from $p(\mathcal{G}; \alpha)$.
Therefore, in practice, we can approximate Eq.~\ref{eq:group_loss_jm1} by a simple MC sampling process:
sample $\alpha$, apply MixUp with rate $\alpha$, compute the loss using the mixed-up samples; repeat for each minibatch in SGD updates.
This simple mechanism enables augmenting data by marginalizing group information; hence the training cannot rely on spurious group patterns because each sample belongs to a "slightly different" group.

\paragraph{Mixing vs Oversampling.}
\begin{figure}
\centering
	\begin{minipage}{.45\columnwidth}
		\centering
		\includegraphics[width=.9\columnwidth]{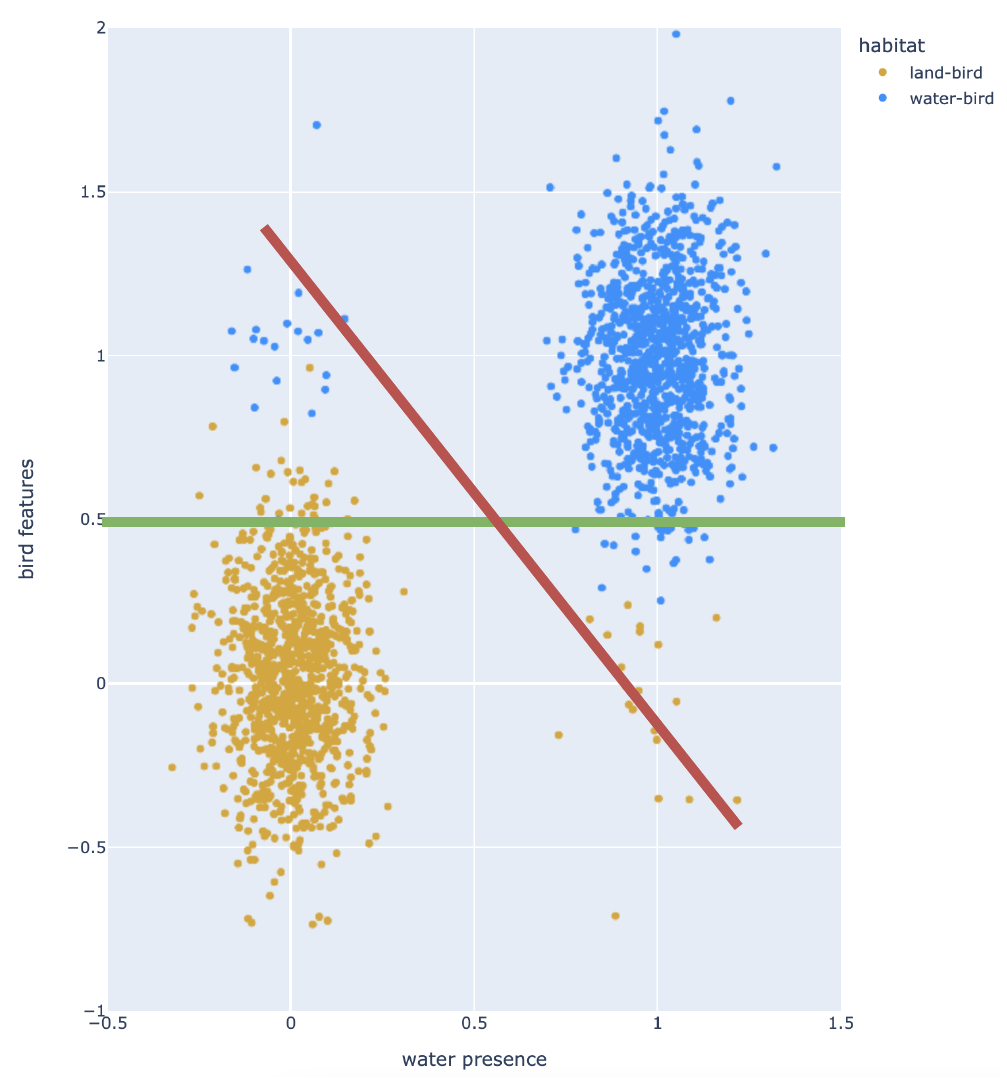}
		\caption{Train.}
		\label{fig:source-domain}
	\end{minipage}%
	\begin{minipage}{.45\columnwidth}
		\centering
		\includegraphics[width=.9\columnwidth]{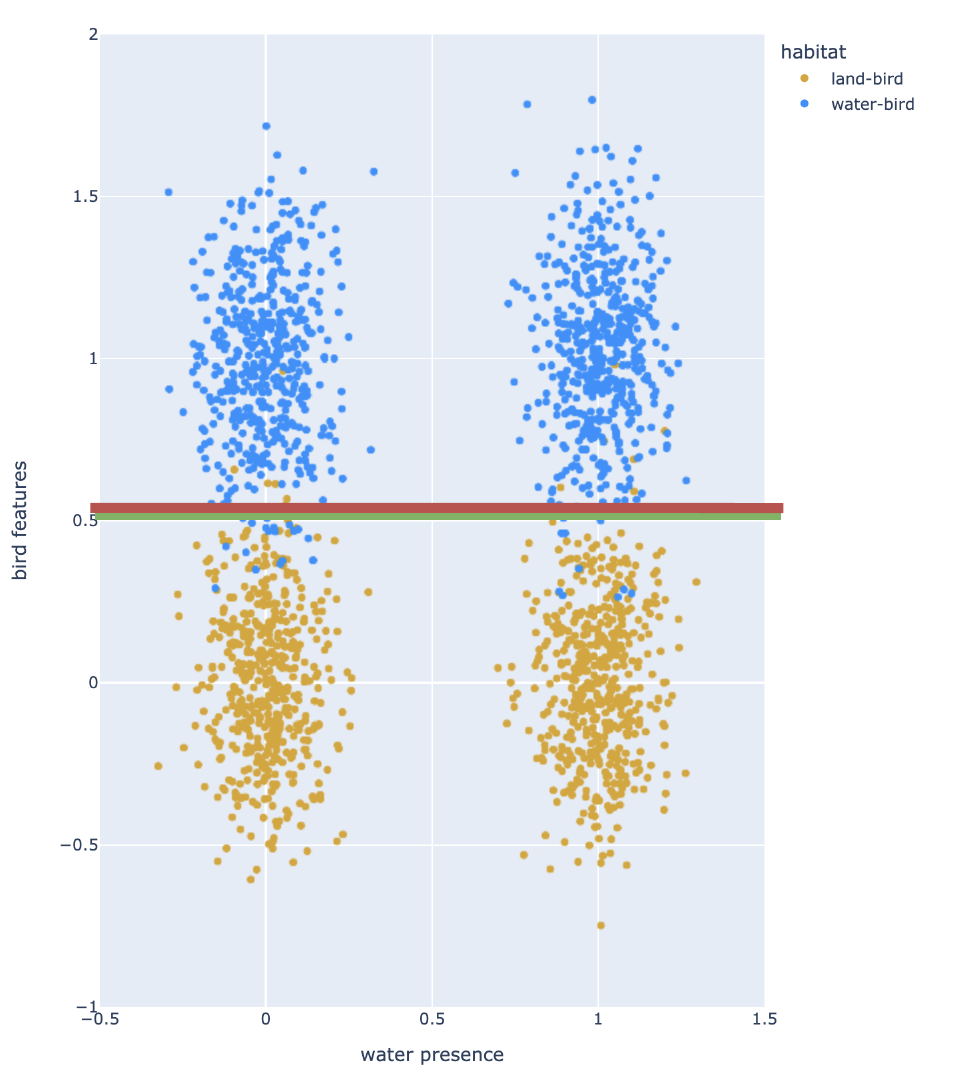}
		\caption{Test.}
		\label{fig:target-domain}
	\end{minipage}
\end{figure}

\begin{figure}
\centering
	\begin{minipage}{.45\columnwidth}
		\centering
		\includegraphics[width=.9\columnwidth]{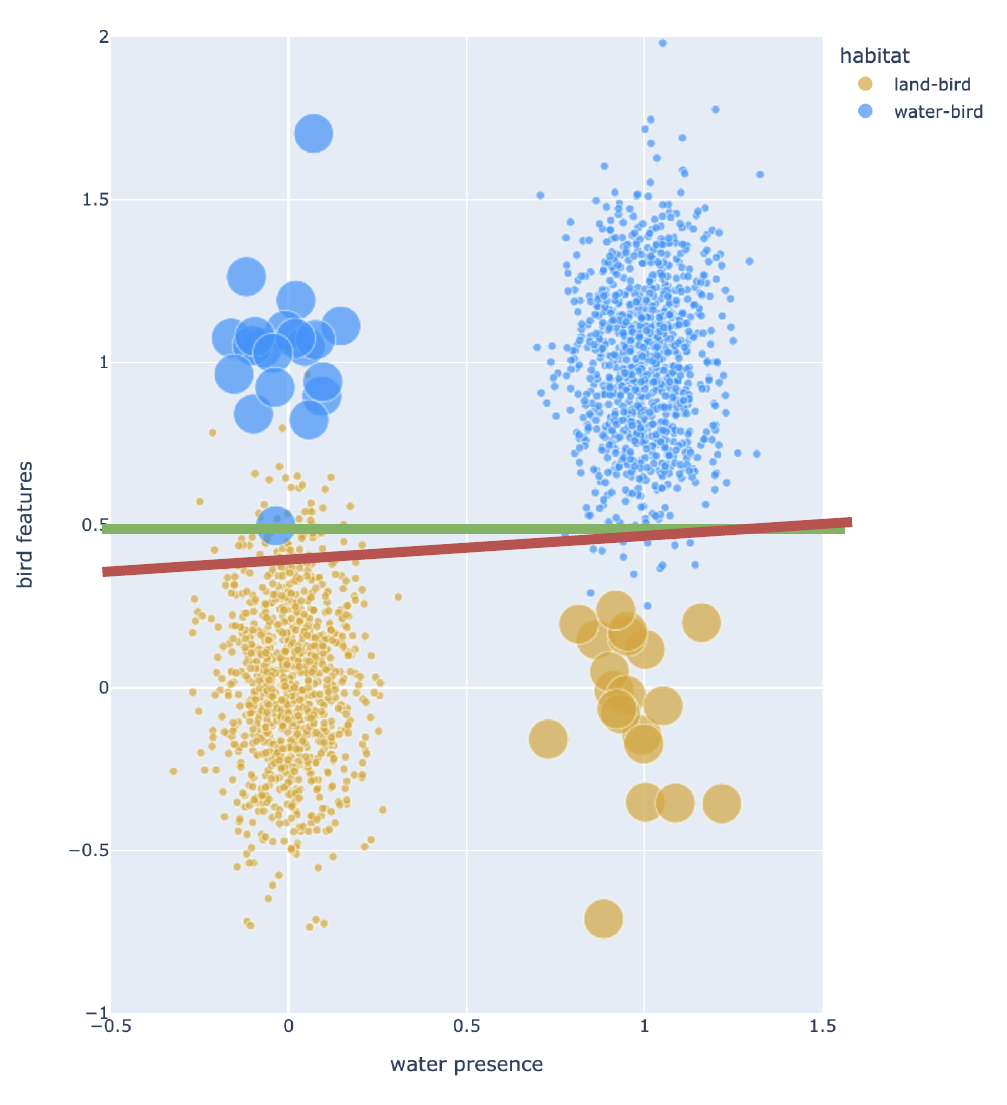}
		\caption{JTT.}
		\label{fig:jtt-oversample}
	\end{minipage}%
	\begin{minipage}{.45\columnwidth}
		\centering
		\includegraphics[width=.9\columnwidth]{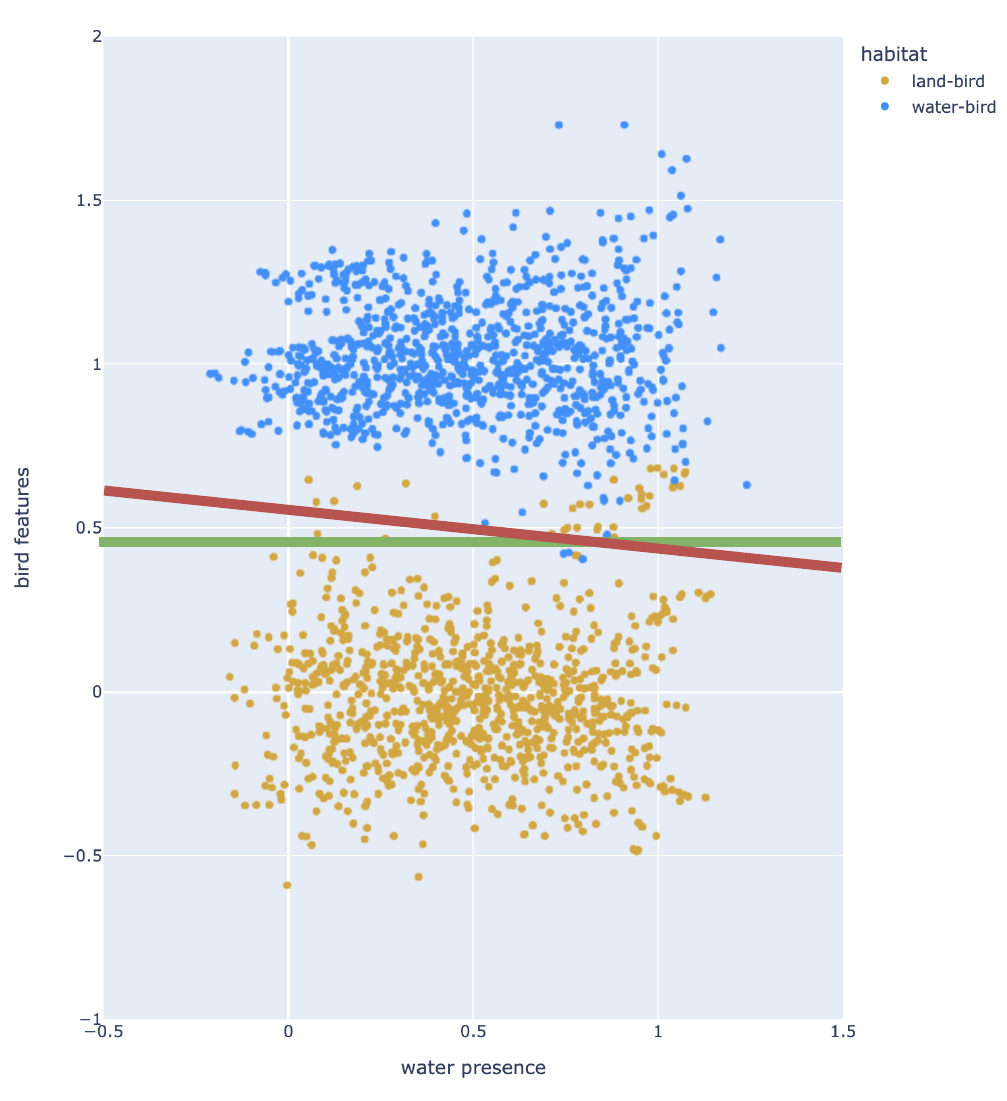}
		\caption{JM1.}
		\label{fig:jm1-mix}
	\end{minipage}
\end{figure}
Here, we discuss a toy visualization to build intuition about the mixing procedure versus a standard oversampling strategy.
Given two classes, \texttt{water-bird} (blue points) and \texttt{land-bird} (brown points), some of the groups are underrepresented in the train set (Fig.~\ref{fig:source-domain}).
In particular, \texttt{land-birds} with \texttt{water-background} and \texttt{water-birds} with \texttt{land-background} (x-axis) are the
underrepresented or minority groups.
On the test set, all groups are equally represented (Fig.~\ref{fig:target-domain}).
JTT oversamples the minority groups (Fig.~\ref{fig:jtt-oversample}) and trains on the augmented distribution. 
JM1 mixes minority and majority groups (Fig.~\ref{fig:jm1-mix}) and trains on the augmented distribution.
The green line is the optimal decision boundary for generalization on the test set. The red lines are the boundaries obtained using ERM (Fig.~\ref{fig:source-domain}), JTT (Fig.~\ref{fig:jtt-oversample}) and JM1 (Fig.~\ref{fig:jm1-mix}).  
We see that both methods generate a training distribution closer to the target in Fig.~\ref{fig:target-domain}.

\paragraph{Group Interpolation.}
Here we provide mathematical intuition of why JM1 is effective.
Consider that group is encoded as a 2d variable $g$ that is represented jointly by two 1d variables confounder $c \in \{0,1\}$ and label $y \in \{0,1\}$.
Suppose that a 2d sample representation is sampled from the conditional distribution $h|c, y \sim N((c, y), \sigma)$ where we denote by $g:=(c,y)$ and set $\sigma=1$ for simplicity.
Since minority groups $(c=0, y=1)$ and $(c=1, y=0)$ are only present in the source domain (Fig.~\ref{fig:source-domain}) but not in the target domain (Fig.~\ref{fig:target-domain}), the decision boundary if naively trained on the source domain will perform very poorly in the target domain.

For a pair of samples from a minority group $g$ (under-represented, for example $g=(0, 1)$) and a majority group ${\bar g}$ (over-represented, for example $\bar g = (1, 1)$) respectively, now that $h \sim N(g, \sigma)$ and ${\bar h} \sim N(\bar g, \sigma)$, we see that the mixed sample with rate $\alpha$ follows an augmented distribution:
\begin{align}
\begin{split}
h_{mix} &= \alpha~h + (1 - \alpha)~{\bar h} \\
h_{mix} &\sim N(\overbrace{(c_{mix}, y)}^{g_{mix}}, \overbrace{\alpha^2 + (1-\alpha)^2}^{\leq 1}),
\end{split}
\label{eq:mix_group}
\end{align}
where $c_{mix} \coloneqq \alpha \ c + (1 - \alpha) \ \bar c$. 
Given a representation space where content (associated with label information $y$) and style (associated with the group confounder $c$) factorize, we can write $h= \mu~f(c)$ and $\bar h= \mu~f(\bar c)$, where $\mu$ is a common content, because we are interpolating samples from different groups but with the same label $y_h = y_{\bar h}$, and $f$ is a mapping from group space to representation space.
If $f$ respects linearity (for example the expectation operator), interpolating samples is equivalent to interpolating groups. More formally:

\begin{theorem}[\textbf{Group Interpolation}]$ $

\begin{itemize}
\item Let \(f\) be a continuous function from group to representation space $f: \mathbb{G} \rightarrow \mathbb{R}^d$ that respects \textbf{linearity} and $g=(c, y)$ a group with label $y$ and factor $c$. 

\item Let $(x, \bar x) \in \mathbb{R}^D$ be samples from the \textbf{same class} $y_x = y_{\bar x}$ but different groups, that is, $x \in g=(c, y)$ and $\bar x \in \bar g=(\bar c, y)$.

\item Let $\mathbb{R}^d$ be a representation space such that content $\mu$ and factor $c$ can be \textbf{factorized} as $h= \mu(y_x)\ f(c)$ and $\bar h= \mu(y_{\bar x})\ f(\bar c)$ where $h$ and $\bar h$ are encoding for $x$ and $\bar x$.
\end{itemize}

$\rightarrow$ Then \underline{$h_{mix} = \mu \ f(c_{mix})$}.
\end{theorem}

\begin{proof}\renewcommand{\qedsymbol}{}
We follow a concise constructive approach:
\begin{align}
\begin{split}
\underline{h_{mix}} : &= 
\alpha~h + (1 - \alpha)~\bar h \\
& = \overbrace{\alpha~\mu(y_x) ~f(c) + (1 - \alpha)~\mu(y_{\bar x}) ~f(\bar c)}^{\texttt{factorization}} \\
& = \overbrace{\alpha~\mu ~f(c) + (1 - \alpha)~\mu ~f(\bar c)}^{\texttt{class-conditional}} \\
& = \mu~(\alpha~f(c) + (1 - \alpha)~f(\bar c)) \\
&= \underbrace{\mu~f(\alpha~ c + (1 - \alpha)~\bar c)}_{\texttt{linearity}} \\
: &= \underline{\mu ~f(c_{mix})},
\end{split}
\end{align}
where $\mu(y_x) = \mu(y_{\bar x}) \coloneqq \mu(y) = \mu$ because $y_x = y_{\bar x} \coloneqq y$.
\end{proof}

\begin{corollary}
Mixing class-conditional samples from groups $g$ and $\bar g$ with rate $\alpha$ is equivalent to drawing samples from group $g_{mix} = (c_{mix}, y) = (\alpha \ c \ + \ (1-\alpha) \ \bar c, y)$ as presented in Eq.~\ref{eq:mix_group}.
\end{corollary}
If $f=I$, we obtain Eq.~\ref{eq:mix_group}. 
We can see that, given a continuous distribution for $\alpha$, $g_{mix}$ is a sample from a continuous distribution that augments the training groups by convex combination.
It is easy to see that a group sampled from the mixing distribution conditional on the same class $y$ (Fig.~\ref{fig:jm1-mix}) will be closer to the target than the source domain, achieving a similar effect of the oversampling approach taken by JTT (Fig.~\ref{fig:jtt-oversample}).
It is also clear that without the class-conditional constraint, mixing will result in an augmented source distribution whose decision boundary will be inconsistent with the target distribution.
A similar idea is visualized in Fig.~\ref{fig:normal_train_test_distro} and Fig.~\ref{fig:augmented_train_test_distro} where we assume that all samples for a group are concentrated in a small region around a normal distribution.

\subsection{GroupJM1}

Note that JM1 only needs access to a coarse group annotation on the train set:
it only requires the oracle majority and minority group assignment for each sample, or infers it based on self-supervised group identification (Phase I).

In the presence of fine-grained group annotation where we have per-sample group labels on train (and validation set), we present a variant of JM1, called GroupJM1, to utilize all available information.
This variant is motivated by GroupDRO.
\citep{sagawa2019distributionally} proposed to approximate the non-smooth DRO loss in Eq.~\ref{eq:dro_loss} by a smooth loss:
$J_{\mathtt{GroupDRO}} = \sum\nolimits^N_{i=1} \sup_g \sum\nolimits^G_{g=1} q_g~l(x_g^i, y_g^i),$
where $q_g$ is a per-group reweighting for the sample loss $l$; 
$q$ is learned iteratively $q_g \leftarrow q_g\exp(l(x_g^i, y_g^i))$ and exponentially increases the importance of samples with large loss, which can quickly converge to supreme loss in the limiting case.

For GroupJM1 we follow a similar procedure: for each batch we find the worst performing group based on the group loss:
\begin{equation*}
l_g = 1/n_g~\sum^{n_g}_{i=1} l(h_g^i, y_g^i), \quad \forall g,
\end{equation*}
and select the worst-group through $wg = \arg\max_g \{l_g\}^G_{g=1}$.
We then mix the samples from the worst performing group with samples from other groups conditional on the same class label, that is, $h_{mix} = \alpha~h_g^i + (1 - \alpha)~h_{wg}^j$. 
Finally, GroupJM1 uses the smooth approximate of the DRO loss with a minor adjustment to assign the group to each mixed sample:
$$
J^{\mathtt{DRO}}_{\mathtt{GroupJM1}} = \sum\nolimits_{h^i_g, g \in \mathcal{G}} \sum\nolimits_{h^j_{wg}, {wg}\in \mathcal{WG}} ~q_{mix}~ l(h_{mix}, y),
$$
where $q_{mix} = q_g$ if $\alpha > 0.5$ or $q_{mix} = q_{wg}$ if $\alpha < 0.5$.
In words, we use the class-conditional mixed data of a pair of samples from a minority (i.e. worst performing) and a majority (i.e. all other) group, but re-weight per-group loss based on the proximity of the mixed sample to the original samples.

\section{Experiments}
\label{sec:experiments}
Our focus for all experiments is to improve the worst-group performance for classification problems in the vision and language domains with or without group annotation at train time and with different levels of annotation granularity.
We mainly compare JM1 to GroupDRO and JTT given the similarity with such methods.

We report experiments on the benchmark proposed in~\citep{sagawa2019distributionally}. The benchmark consists of two vision and one language dataset (Fig.~\ref{fig:datasets}).
In these datasets, the model can easily exploit superficial, majority group patterns (background in CUB, sex in CelebA, negation in MultiNLI) to solve the classification tasks at train time. 
When tested on a set with a different group distribution (same groups but represented in significantly different proportions), the model performance drops, showing a lack of generalization capacity.

\paragraph{Settings.}
\begin{figure}
	\centering
	\includegraphics[width=0.45\textwidth]{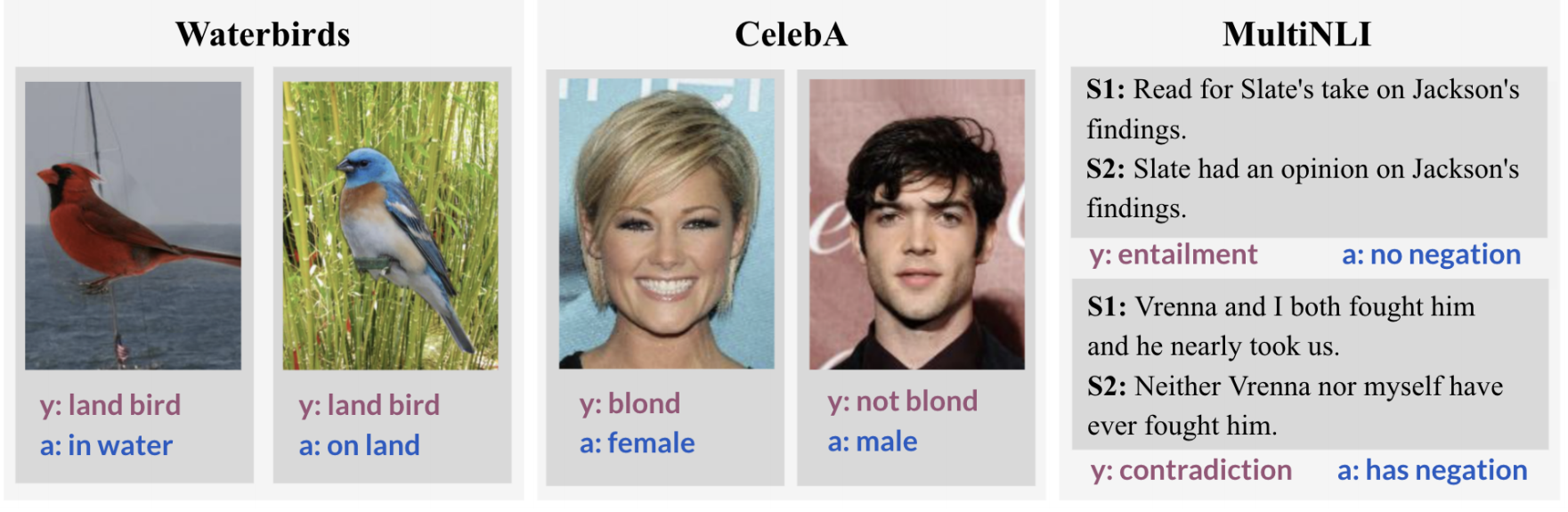}
	\caption{Figure adapted from~\citep{liu2021just}. Examples from the benchmarks. Label information $y$ is spuriously correlated with confounding factors $c$. Zoom for best visualization.
	}
	\label{fig:datasets}
\end{figure}
We initially found that mixing groups by sampling $\texttt{Beta}(1, 1)$\footnote{Samples from $\texttt{Beta}(1, 1)$ are uniform samples from $\texttt{U}(0, 1)$.} uniformly, even though improving worst-group accuracy over ERM, underperforms compared to JTT.
Therefore, we use a mixing strategy with a slight emphasis on the region closer to the minority groups and propose a coupled strategy:
we sample $\texttt{Beta}(1, 1)$ for half of the epochs, and $\texttt{Beta}(2, 5)$ for the other half. 
This simple heuristic ensures that the full mixing domain is spanned, with a focus on the region closer to the misclassified samples in phase I). 
We used this coupled sampling strategy for CUB and CelebA.
For MultiNLI, we sample $\alpha$ only from the uniform prior because the mixing phase runs for a very small number of epochs.
More details of JM1 training are described in Appendix~\ref{appendix:details}.

It is worth noting that: I) like JTT, JM1 is domain-agnostic and can be applied to vision and language datasets with minimal architectural changes; II) unlike JTT's oversampling approach, JM1 introduces an interpolation operation between samples.
Note that we can apply JM1 to mix samples at different levels of abstraction: either the input~\citep{zhang2017mixup} or the representation space~\citep{verma2019manifold}.
As the comparison of different mixup levels is not the focus of our study, we refer the readers to Appendix~\ref{appendix:experiments} for additional ablations.

\paragraph{Datasets.}
\underline{\texttt{Waterbirds (CUB)}}. 
The task is to classify land birds $\vy_l$ and water birds $\vy_w$.
The label information is correlated with a confounding factor: land background $\vc_l$ and water background $\vc_w$.
The majority groups are (land bird, land background) - $(\vy_l, \vc_l)$ and (water bird, water background) - $(\vy_w, \vc_w)$. 
The minority groups are (land bird, water background) - $(\vy_l, \vc_w)$ and (water bird, land background) - $(\vy_w, \vc_l)$. 
During training, minority groups are present in $5 \%$ of the data. 
At test time the groups are evenly distributed.

\underline{\texttt{CelebA}}.
The task is to classify blond person $\vy_b$ and not blond person $\vy_{nb}$.
The label information is correlated with a confounding factor: male $\vc_m$ and female $\vc_f$.
The majority groups are (not blond, male) - $(\vy_{nb}, \vc_m)$ and (blond, female) - $(\vy_b, \vc_f)$. The minority groups are (not blond, female) - $(\vy_{nb}, \vc_f)$ and (blond, male) - $(\vy_b, \vc_m)$. 

\underline{\texttt{MultiNLI}}.
The task is to classify an inference procedure as entailment $\vy_e$, contradiction $\vy_c$ and neutral $\vy_n$.
The label information is correlated with a confounding factor: the presence of a negation $\vc_{neg}$ and the absence of a negation $\vc_{nneg}$ in the reasoning.

\subsection{Worst-Group Generalization} 
To evaluate the generalization capacity of our approach, we consider 3 different scenarios in terms of group annotation. 

1. \underline{Fine-grained} group annotation on the train set.

2. \underline{Coarse} group annotation on the train set.

3. \underline{Validation set} group annotated.

\paragraph{Fine-grained group annotation on the train set.}
\begin{table}
\small
	\begin{center}
		\caption{\textbf{Fine-grained group annotation}.
		Experiments using GroupDRO and GroupJM1 variants. 
			We assume fine-grained group annotation on the train set. 
			We find the worst performing group in each mini batch and use such information to improve JM1. We samples $\alpha \sim \texttt{U}(0, 1)$. GroupDRO results from~\citep{sagawa2019distributionally} and LISA results from~\citep{yao2022improving}.
			LISA and JM1 use a ERM loss as training objective. GroupDRO and GroupJM1 use the DRO loss.
			}
			\setlength\tabcolsep{4.0pt}
		\begin{tabular}{l c c c c c c} 
			\toprule
			& \multicolumn{2}{c}{CUB} & \multicolumn{2}{c}{CelebA} & \multicolumn{2}{c}{MultiNLI} \\ %
			\vspace{2pt}
			& avg & \underline{worst}  & avg & \underline{worst} & avg & \underline{worst}   \\ %
			LISA     & 91.8   & 89.2 & 92.4 & 89.3 & -& - \\ %
			GroupDRO &93.5& \textbf{91.4} &92.9&88.9 &81.4&77.7 \\ %
			\textbf{JM1} (ours)      &90.9& 90.5& 92.0&86.1&82.1&\textbf{81.5} \\ %
			\textbf{GroupJM1} (ours) &93.3& \textbf{91.3}&93.0&\textbf{90.0}  &81.9& 79.3 \\ %
			\bottomrule
		\end{tabular}
	\label{table:group_jm1}
	\end{center}
\end{table}
In the presence of multiple groups in the data, we have access to fine-grained group annotation on the train set. Consequently, group-aware methods can be employed.
To test JM1 using oracle groups (i.e. without the Group Identification phase I) and its variant GroupJM1 (end Sec.\ref{sec:jm1}), which exploit known group annotation during training, we compare them with GroupDRO~\citep{sagawa2019distributionally} and LISA~\citep{yao2022improving}.
Note that LISA exploits a mixture of intra-label and intra-domain mixup with fine-grained multi-class group annotation on the train set, while JM1 exploits class-conditional mixup using a majority vs minority binary group assignment.
GroupDRO finds the worst-performing group in each batch, computing the per-group loss and reweighting the training loss using ad-hoc hyperparameters to adjust for minority groups.
As shown in Table~\ref{table:group_jm1}, GroupJM1 is very competitive in this scenario. Particularly in the MultiNLI data set, JM1 is the best performing method even using a simple ERM on the mixed sample, corroborating the class-conditional mixing procedure as a general approach to improve worst-group accuracy. 

\paragraph{Coarse group annotation on the train set.}
\begin{table}
\small
		\begin{center}
			\caption{\textbf{Coarse group annotation}. Results on CUB and CelebA datasets with coarse group annotation, where we use oracle majority and minority group annotations (instead of trained groups from Phase I).
			JM1 with class-conditional mixing generalizes better than oversampling-based JTT.}
		\begin{tabular}{l c c c c} 
			\toprule
			& \multicolumn{2}{c}{CUB} & \multicolumn{2}{c}{CelebA} \\
			\vspace{2pt}
			& avg & \underline{worst} & avg & \underline{worst} \\
			JTT (w/o I)  & 96.7         & 75.9    &93.4 &57.2\\
			\textbf{JM1} (w/o I) & 95.0         &\textbf{86.1} &92.3 &\textbf{85.6}\\
			\bottomrule
		\end{tabular}
		\label{table:oracle}
	\end{center}
\end{table}
The second set of experiments compares JM1 and JTT in a different scenario: we assume that we have access to the ground-truth minority group assignment on the train set.
Note that \emph{this does not equate} to the fine-grained per-sample group annotation: 
we only need a binary mask to know which sample has a group-conflicting confounder against its label.
Now we can use such an oracle majority vs. minority groups in place of the identified set from phase I).
As noted in Appendix C.2 in~\citep{liu2021just}, the performance of JTT drops if not using identified error sets from phase I), since JTT can excessively oversample the minority group.
Given its inner working of mixing groups rather than oversampling, we expect JM1 to outperform JTT in this oracle configuration.
Table~\ref{table:oracle} validates our intuition.
JM1 outperforms JTT (oracle) on both datasets.

\paragraph{Group annotation on the validation set.}

\begin{table}
\centering
	\caption{\textbf{Validation-set per-group accuracy}. 
	Per-group accuracy on \underline{CUB} dataset.
	    The goal is to classify land birds $\vy_l$ vs water birds $\vy_w$ in the presence of a confounders: land background $\vc_l$ vs water background $\vc_w$.
		JM1 is competitive with JTT improves accuracy on both minority groups $(\vy_l, \vc_w)$ and $(\vy_w, \vc_l)$ and competitive on the other two majority groups.
		We select the best run among 5 in terms of worst-group accuracy.}
		\setlength\tabcolsep{4.0pt}
		\small
	\begin{tabular}{l c c c c c c } 
		\toprule
		&\multicolumn{4}{c}{CUB} & &\\ 
		\vspace{2pt}
		& ($\vy_l$, $\vc_l$) & ($\vy_l$, $\vc_w$) &  ($\vy_w$, $\vc_l$) & ($\vy_w$, $\vc_w$) & avg & \underline{worst} \\
		JTT   & 94.3 & 86.7 & 87.5 & 91.6  & 93.3& 86.7\\
		\textbf{JM1} & 94.2 & 88.5 &88.9& 90.5 & 93.1 & \textbf{88.5}\\
		\bottomrule
	\end{tabular}
	\label{table:cub}
\end{table}
\begin{table}
\centering
		\caption{\textbf{Validation-set per-group accuracy}. 
		Per-group accuracy on \underline{CelebA} dataset. 
		The goal is to classify not blond $\vy_{nb}$ vs blond $\vy_b$ persons in the presence of a confounder: male $\vc_m$ vs female $\vc_f$.
		We select the best run among 5 in terms of worst-group accuracy.}
		\setlength\tabcolsep{4.0pt}
		\small
		\begin{tabular}{l c c c c c c } 
			\toprule
			&\multicolumn{4}{c}{CelebA} & &\\ 
			\vspace{2pt}
			&  ($\vy_{nb}$, $\vc_m$) & ($\vy_{nb}$, $\vc_f$) & ($\vy_{b}$, $\vc_m$) &  ($\vy_b$, $\vc_f$) & avg & \underline{worst} \\
			JTT   & 90.9 & 87.9 & 82.8 & 89.0 & 89.1&82.8 \\
			\textbf{JM1} & 88.8&87.7&84.4&86.6 &87.9&\textbf{84.4} \\
			\bottomrule
		\end{tabular}
	\label{table:celeba}
\end{table}
The last set of experiments evaluate JM1 without group annotation on the train set but with an annotated small validation set. 
This is the configuration proposed in JTT.
The models have access to a validation set with group annotation.
The main goal of JM1 is to improve the classification accuracy of minority groups by mixing samples from different groups with the same class.
Table~\ref{table:cub_sota} compares our approach with JTT. We select the best identification strategy for JTT and then use these samples for the second phase. 
As for JTT, we use an annotated validation set to select the best model.
We see that our approach is competitive with SOTA in improving the worst-group performance in all three datasets with comparable average accuracy.

Our approach is scalable, involving an additional forward pass and a cheap interpolation between samples, adding minimal computation overhead; and domain agnostic because it can be applied to vision and language datasets.
With JM1 we remove the need for oversampling and $\lambda$. %
In Tables~\ref{table:cub} and ~\ref{table:celeba} we report the per-group accuracy on CUB and CelebA for JTT (retrained) and JM1 with the same setting. We see that not only JM1 does improve worst-group accuracy, but improves accuracy on both minority groups in the dataset with almost no performance loss for the majority groups compared to JTT. 
Overall the results in Table~\ref{table:cub_sota}, Table~\ref{table:cub} and Table~\ref{table:celeba} are a clear indication that JM1 is competitive with JTT, and can be used to deal with worst-group performance in an effective manner.
\begin{table*}[ht]
\small
	\begin{center}
		\caption{\textbf{Validation-set group annotation}. 
		Results on CUB, CelebA and MultiNLI datasets for average accuracy and worst-group accuracy.
			JM1 performs comparably or better than the SOTA in worst-group performance without group annotation on the train-set, closing the gap with GroupDRO that requires full group annotations on the train set.
			\textbf{1st} / \underline{2nd} best worst-group accuracies \textbf{bolded} / \underline{underlined}.
			Our method is better or comparable with sota with the advantage of being computationally inexpensive, used with different level of annotation, different data modalities, without introducing regularization or additional complexity.}
		\begin{tabular}{l c c c c c c c} 
			\toprule
			& \multicolumn{2}{c}{CUB} & \multicolumn{2}{c}{CelebA} & \multicolumn{2}{c}{MultiNLI} & group \\
			\vspace{2pt}
			& avg acc & \underline{worst acc} & avg acc & \underline{worst acc} & avg acc & \underline{worst acc} & labels \\
			ERM~\citep{liu2021just}                        &  97.6   & 72.6 & 95.6 & 47.2  & 82.4 & 67.9 & \xmark \\
			CVaR-DRO~\citep{levy2020large}         & 96.5    & 69.5 &82.4  & 64.4 &82.0  &68.0 & \xmark \\
			LfF~\citep{nam2020learning}                & 97.3    & 75.2 & 86.0 & 70.6  & 80.8 & \underline{70.2} & \xmark \\
			EIIL~\citep{creager2021environment}    & 96.9    & 78.6 & 85.7 & 81.7&-&- & \xmark \\
			\vspace{2pt}
			JTT~\citep{liu2021just}                         & 93.3    & \underline{86.7} &  89.1 & 82.8  &78.6   &\textbf{72.6} & \xmark \\
			CNC~\citep{zhang2022contrastive} & 90.9 & \textbf{88.4} & 89.9 & \textbf{88.8} & - & - & \xmark \\
            \textbf{JM1 (ours)} & 93.1 & \textbf{88.5}  & 87.9 & \underline{84.4} & 80.3& \textbf{72.5} & \xmark \\
			\midrule
			GroupDRO~\citep{sagawa2019distributionally}& 93.5& 91.4& 92.9& 88.9& 81.4& 77.7& \cmark \\
			\bottomrule
		\end{tabular}
	\label{table:cub_sota}
	\end{center}
\end{table*}

\subsection{JM1 Ablations}
We compare the proposed JM1 with standard (unconditional) MixUp~\citep{zhang2017mixup} and a class-conditional MixUp baseline, where we mix samples from random groups.
Note that the class-conditional MixUp baseline mixes samples from \textit{any} two different groups conditional on the same class, while JM1 only mixes samples from a majority and a minority group conditional on the same class.
\begin{figure}[ht]
		\begin{center}
		\includegraphics[width=0.7\linewidth]{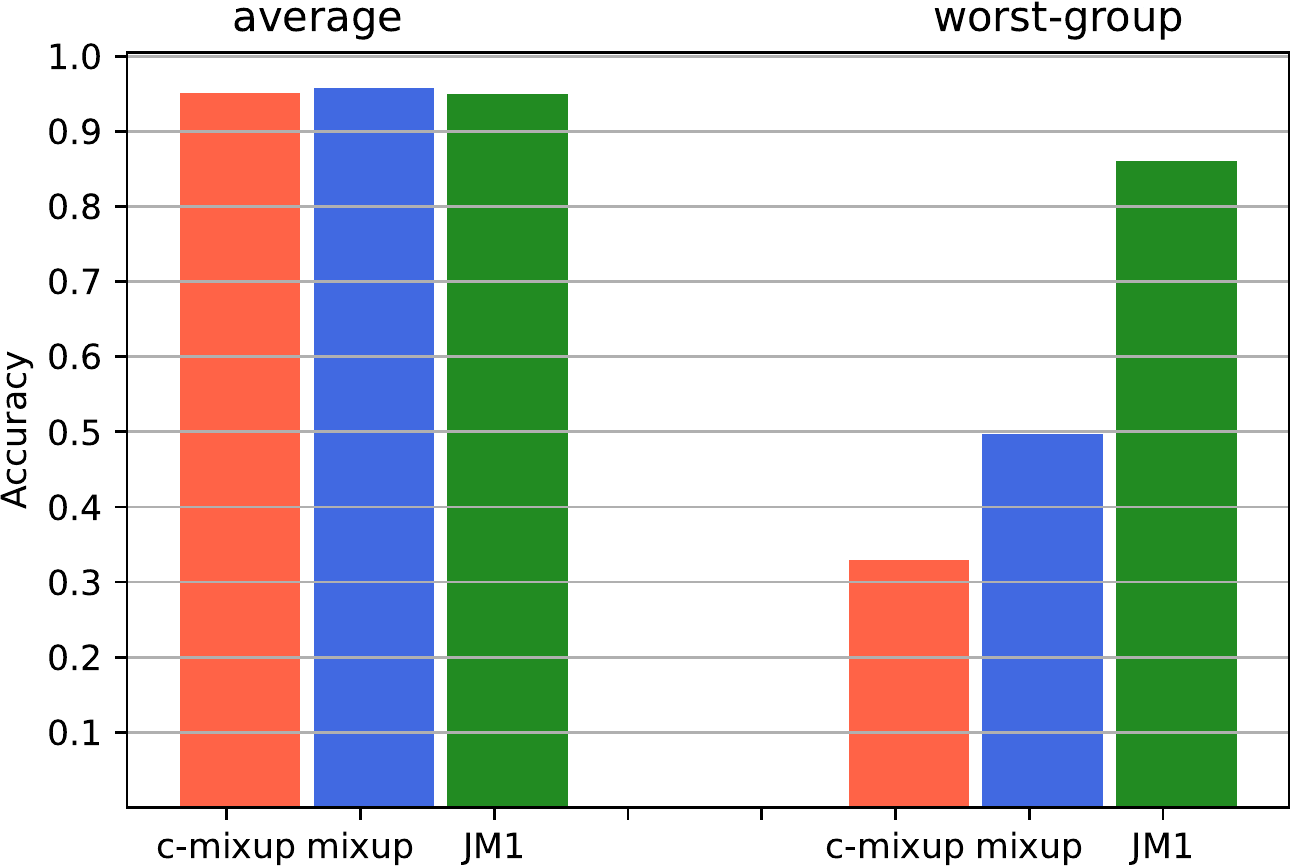}
		\caption{\textbf{MixUp Ablation}. 
		Results on CUB dataset JM1 vs class-conditional (c-MixUp) and standard MixUp baselines. 
		Note that standard MixUp fails to improve worst-group accuracy.}
			\label{bar:oracle}
		\end{center}
\end{figure}
Fig.~\ref{bar:oracle} shows that, while standard MixUp fails to improve worst-group performance, the proposed mixing strategy performs significantly better than both MixUp baselines by a large margin, despite that all three methods perform very similarly in terms of average performance.
In Fig.~\ref{figure:gap} we investigate the performance of JM1 with and without the class-conditional constraint. 
We see that empirically, conditional JM1 performs better than unconditional JM1.
\begin{figure}[ht]
		\begin{center}
		\includegraphics[width=0.7\linewidth]{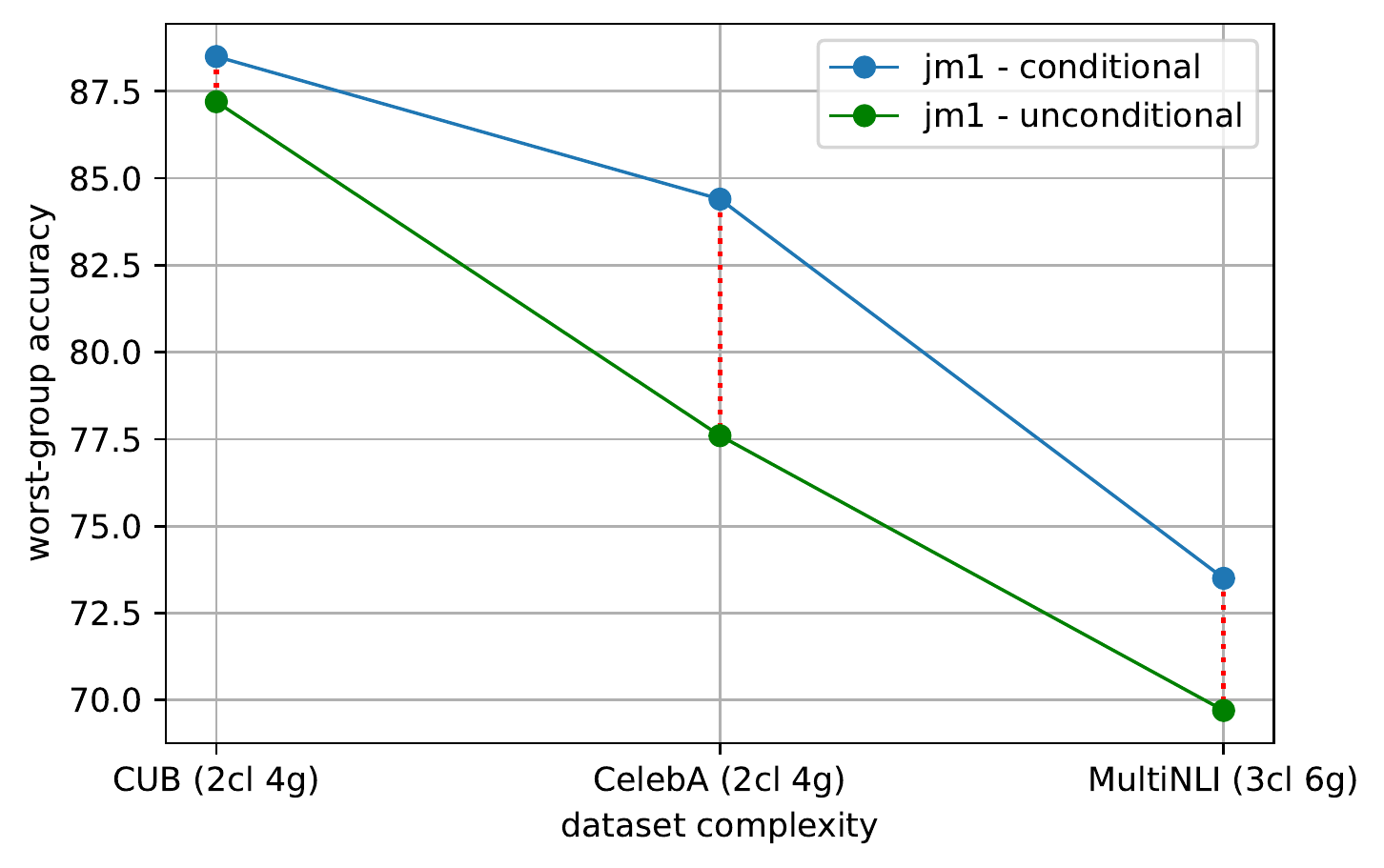}
		\caption{\textbf{JM1 Ablation}. 
		Unconditional vs conditional JM1 on  the three benchmark datasets (CUB, CelebA, MultiNLI) sorted by complexity. 
		The worst-group performance decreases as the dataset complexity (dataset size, number of groups, variety of the data) increases.
		}
		\label{figure:gap}
	\end{center}
\end{figure}
The worst-group accuracy decreases as the dataset complexity (in terms of the number of samples, classes, and groups) increases for both models.
\section{Related Work}	

\paragraph{Robust classification.}
Optimizing average accuracy with overparametrized neural networks tends to get rid of minority group information. 
This fact is related to how the model is trained and the way the model exploits superficial correlations in the data~\citep{geirhos2018imagenet}.
Recently approaches to preserve group information have been proposed, leveraging robust optimization~\citep{sagawa2019distributionally, mohri2019agnostic, zhang2020coping}, data augmentation~\citep{goel2020model}, debiasing approaches~\citep{hong2021unbiased, bahng2020learning}, and group reweighting~\citep{shimodaira2000improving, byrd2019effect, sagawa2020investigation} using importance samples to rebalance minority and majority groups. 
If full group annotation is provided, robust optimization methods are employed. The supervised scenario is closely related to the fairness one~\citep{hardt2016equality, agarwal2018reductions, oneto2020fairness, pleiss2017fairness, woodworth2017learning}: wherewith distribution shift we want to improve the worst-group generalization, in fairness we want to learn features and distributions that are invariant to different groups, using explicit constraints~\citep{zafar2017fairness, zafar2019fairness, zemel2013learning}, invariant learning~\citep{arjovsky2019invariant}, auxiliary tasks~\citep{edwards2015censoring, lahoti2020fairness}, and causal structure~\citep{oneto2020fairness}.
In the causality literature~\citep{pearl2009causality} the group information is interpreted as a confounder. 
If we have access to the confounder (group annotation), we can control for it.
Recently risk extrapolation~\citep{krueger2021out} has been proposed to deal with domain shift in the presence of confounders.
Similarly, JM1 augments the training data with plausible scenarios not present in the data. A similar approach has been proposed in~\citep{chuang2021fair} in the fairness domain and in~\citep{yao2022improving} for group robustness with full group annotation.

\paragraph{Groups generalization without annotation.}
All the aforementioned methods assume that fine-grained group annotation is available at train time.
However, assuming subpopulation annotation in large diverse datasets is problematic because the process is time consuming and prone to errors. 
Learning from Failure~\citep{nam2020learning} proposes a group annotation-free and domain agnostic procedure to improve worst-group performance. The authors propose a self-supervised approach to identify bias-aligned (majority) vs. bias-contrastive (minority) groups in the data.
The self-supervised signal consists of the per-sample loss magnitude in the early stage of training;
Then the dataset is re-weighted based on the loss magnitude and the procedure applied multiple times during training. Similar ideas are explored in~\citep{yaghoobzadeh2019increasing, utama2020towards}.
Just Train Twice~\citep{liu2021just} proposes a similar solution, in which the identification phase is performed one time and then a reweighting phase is trained. %
\citep{sohoni2020no} uses clustering to improve the classification of underrepresented classes. 
Distributional robust optimization~\citep{ben2013robust, duchi2019distributionally, duchi2021statistics}, where the optimization problem is defined over a set of distributions around the empirical distribution, has also been widely used to improve group robustness. Invariant Learning~\citep{arjovsky2019invariant, creager2021environment} is another approach to improve worst-group performance. The goal is to obtain a learner that performs well in different environments, where in each environment we have the same classes but different confounders. Concurrent to our work, SelecMix~\citep{hwang2022selecmix} proposes to identify groups using contrastive learning and then use mixup similarly to the second phase of JM1. However, the method is applied only to image data, in one specific setting (validation set group annotation), and relies on distance between samples, where JM1 does not make assumptions on the data modality or the specific setting.

\section{Conclusion}
We present JM1, a simple and effective method to improve worst-group performance when different level of annotation are available. 
We show that our approach is competitive with sota and robust to the choice of identification phase. Class-conditional group interpolation provides a cheap, domain-agnostic way to augment the training distribution and improve worst-group accuracy.
However, our approach does not guarantee fairness.
We cannot use JM1 when there is perfect alignment between confounders and labels. 
We consider the data-driven problem without accounting for the collection, labeling, and source of the data. 
The results are sensitive to the sampling distribution for $\alpha$. 
More robust interpolations mechanisms will be explored in future work.

\clearpage
\bibliographystyle{apalike}
\bibliography{biblio.bib}

\begin{thebibliography}{}

\bibitem[Agarwal et~al., 2018]{agarwal2018reductions}
Agarwal, A., Beygelzimer, A., Dud{\'\i}k, M., Langford, J., and Wallach, H.
  (2018).
\newblock A reductions approach to fair classification.
\newblock In {\em International Conference on Machine Learning}, pages 60--69.
  PMLR.

\bibitem[Arazo et~al., 2019]{arazo2019unsupervised}
Arazo, E., Ortego, D., Albert, P., O’Connor, N., and McGuinness, K. (2019).
\newblock Unsupervised label noise modeling and loss correction.
\newblock In {\em International Conference on Machine Learning}, pages
  312--321. PMLR.

\bibitem[Arjovsky et~al., 2019]{arjovsky2019invariant}
Arjovsky, M., Bottou, L., Gulrajani, I., and Lopez-Paz, D. (2019).
\newblock Invariant risk minimization.
\newblock {\em arXiv preprint arXiv:1907.02893}.

\bibitem[Bahng et~al., 2020]{bahng2020learning}
Bahng, H., Chun, S., Yun, S., Choo, J., and Oh, S.~J. (2020).
\newblock Learning de-biased representations with biased representations.
\newblock In {\em International Conference on Machine Learning}, pages
  528--539. PMLR.

\bibitem[Ben-Tal et~al., 2013]{ben2013robust}
Ben-Tal, A., Den~Hertog, D., De~Waegenaere, A., Melenberg, B., and Rennen, G.
  (2013).
\newblock Robust solutions of optimization problems affected by uncertain
  probabilities.
\newblock {\em Management Science}, 59(2):341--357.

\bibitem[Byrd and Lipton, 2019]{byrd2019effect}
Byrd, J. and Lipton, Z. (2019).
\newblock What is the effect of importance weighting in deep learning?
\newblock In {\em International Conference on Machine Learning}, pages
  872--881. PMLR.

\bibitem[Carratino et~al., 2020]{carratino2020mixup}
Carratino, L., Ciss{\'e}, M., Jenatton, R., and Vert, J.-P. (2020).
\newblock On mixup regularization.
\newblock {\em arXiv preprint arXiv:2006.06049}.

\bibitem[Chuang and Mroueh, 2021]{chuang2021fair}
Chuang, C.-Y. and Mroueh, Y. (2021).
\newblock Fair mixup: Fairness via interpolation.
\newblock {\em arXiv preprint arXiv:2103.06503}.

\bibitem[Creager et~al., 2021]{creager2021environment}
Creager, E., Jacobsen, J.-H., and Zemel, R. (2021).
\newblock Environment inference for invariant learning.
\newblock In {\em International Conference on Machine Learning}, pages
  2189--2200. PMLR.

\bibitem[Duchi et~al., 2021]{duchi2021statistics}
Duchi, J.~C., Glynn, P.~W., and Namkoong, H. (2021).
\newblock Statistics of robust optimization: A generalized empirical likelihood
  approach.
\newblock {\em Mathematics of Operations Research}.

\bibitem[Duchi et~al., 2019]{duchi2019distributionally}
Duchi, J.~C., Hashimoto, T., and Namkoong, H. (2019).
\newblock Distributionally robust losses against mixture covariate shifts.
\newblock {\em Under review}.

\bibitem[Edwards and Storkey, 2015]{edwards2015censoring}
Edwards, H. and Storkey, A. (2015).
\newblock Censoring representations with an adversary.
\newblock {\em arXiv preprint arXiv:1511.05897}.

\bibitem[Geirhos et~al., 2017]{geirhos2017comparing}
Geirhos, R., Janssen, D.~H., Sch{\"u}tt, H.~H., Rauber, J., Bethge, M., and
  Wichmann, F.~A. (2017).
\newblock Comparing deep neural networks against humans: object recognition
  when the signal gets weaker.
\newblock {\em arXiv preprint arXiv:1706.06969}.

\bibitem[Geirhos et~al., 2018a]{geirhos2018imagenet}
Geirhos, R., Rubisch, P., Michaelis, C., Bethge, M., Wichmann, F.~A., and
  Brendel, W. (2018a).
\newblock Imagenet-trained cnns are biased towards texture; increasing shape
  bias improves accuracy and robustness.
\newblock {\em arXiv preprint arXiv:1811.12231}.

\bibitem[Geirhos et~al., 2018b]{geirhos2018generalisation}
Geirhos, R., Temme, C.~R., Rauber, J., Sch{\"u}tt, H.~H., Bethge, M., and
  Wichmann, F.~A. (2018b).
\newblock Generalisation in humans and deep neural networks.
\newblock {\em Advances in neural information processing systems}, 31.

\bibitem[Goel et~al., 2020]{goel2020model}
Goel, K., Gu, A., Li, Y., and R{\'e}, C. (2020).
\newblock Model patching: Closing the subgroup performance gap with data
  augmentation.
\newblock {\em arXiv preprint arXiv:2008.06775}.

\bibitem[Hardt et~al., 2016]{hardt2016equality}
Hardt, M., Price, E., and Srebro, N. (2016).
\newblock Equality of opportunity in supervised learning.
\newblock {\em Advances in neural information processing systems},
  29:3315--3323.

\bibitem[Hastie et~al., 2009]{hastie2009elements}
Hastie, T., Tibshirani, R., Friedman, J.~H., and Friedman, J.~H. (2009).
\newblock {\em The elements of statistical learning: data mining, inference,
  and prediction}, volume~2.
\newblock Springer.

\bibitem[Hendrycks et~al., 2021]{hendrycks2021many}
Hendrycks, D., Basart, S., Mu, N., Kadavath, S., Wang, F., Dorundo, E., Desai,
  R., Zhu, T., Parajuli, S., Guo, M., et~al. (2021).
\newblock The many faces of robustness: A critical analysis of
  out-of-distribution generalization.
\newblock In {\em Proceedings of the IEEE/CVF International Conference on
  Computer Vision}, pages 8340--8349.

\bibitem[Hendrycks and Dietterich, 2019]{hendrycks2019benchmarking}
Hendrycks, D. and Dietterich, T. (2019).
\newblock Benchmarking neural network robustness to common corruptions and
  perturbations.
\newblock {\em arXiv preprint arXiv:1903.12261}.

\bibitem[Hermann et~al., 2020]{hermann2020origins}
Hermann, K., Chen, T., and Kornblith, S. (2020).
\newblock The origins and prevalence of texture bias in convolutional neural
  networks.
\newblock {\em Advances in Neural Information Processing Systems},
  33:19000--19015.

\bibitem[Hong and Yang, 2021]{hong2021unbiased}
Hong, Y. and Yang, E. (2021).
\newblock Unbiased classification through bias-contrastive and bias-balanced
  learning.
\newblock {\em Advances in Neural Information Processing Systems}, 34.

\bibitem[Hwang et~al., 2022]{hwang2022selecmix}
Hwang, I., Lee, S., Kwak, Y., Oh, S.~J., Teney, D., Kim, J.-H., and Zhang,
  B.-T. (2022).
\newblock Selecmix: Debiased learning by mixing up contradicting pairs.

\bibitem[Ilyas et~al., 2019]{ilyas2019adversarial}
Ilyas, A., Santurkar, S., Tsipras, D., Engstrom, L., Tran, B., and Madry, A.
  (2019).
\newblock Adversarial examples are not bugs, they are features.
\newblock {\em Advances in neural information processing systems}, 32.

\bibitem[James et~al., 2013]{james2013introduction}
James, G., Witten, D., Hastie, T., and Tibshirani, R. (2013).
\newblock {\em An introduction to statistical learning}, volume 112.
\newblock Springer.

\bibitem[Krueger et~al., 2021]{krueger2021out}
Krueger, D., Caballero, E., Jacobsen, J.-H., Zhang, A., Binas, J., Zhang, D.,
  Le~Priol, R., and Courville, A. (2021).
\newblock Out-of-distribution generalization via risk extrapolation (rex).
\newblock In {\em International Conference on Machine Learning}, pages
  5815--5826. PMLR.

\bibitem[Lahoti et~al., 2020]{lahoti2020fairness}
Lahoti, P., Beutel, A., Chen, J., Lee, K., Prost, F., Thain, N., Wang, X., and
  Chi, E.~H. (2020).
\newblock Fairness without demographics through adversarially reweighted
  learning.
\newblock {\em arXiv preprint arXiv:2006.13114}.

\bibitem[Levy et~al., 2020]{levy2020large}
Levy, D., Carmon, Y., Duchi, J.~C., and Sidford, A. (2020).
\newblock Large-scale methods for distributionally robust optimization.
\newblock {\em arXiv preprint arXiv:2010.05893}.

\bibitem[Liu et~al., 2021]{liu2021just}
Liu, E.~Z., Haghgoo, B., Chen, A.~S., Raghunathan, A., Koh, P.~W., Sagawa, S.,
  Liang, P., and Finn, C. (2021).
\newblock Just train twice: Improving group robustness without training group
  information.
\newblock In {\em International Conference on Machine Learning}, pages
  6781--6792. PMLR.

\bibitem[Mohri et~al., 2019]{mohri2019agnostic}
Mohri, M., Sivek, G., and Suresh, A.~T. (2019).
\newblock Agnostic federated learning.
\newblock In {\em International Conference on Machine Learning}, pages
  4615--4625. PMLR.

\bibitem[Nam et~al., 2020]{nam2020learning}
Nam, J., Cha, H., Ahn, S., Lee, J., and Shin, J. (2020).
\newblock Learning from failure: Training debiased classifier from biased
  classifier.
\newblock {\em arXiv preprint arXiv:2007.02561}.

\bibitem[Oneto and Chiappa, 2020]{oneto2020fairness}
Oneto, L. and Chiappa, S. (2020).
\newblock Fairness in machine learning.
\newblock {\em Recent Trends in Learning From Data}, pages 155--196.

\bibitem[Pearl, 2009]{pearl2009causality}
Pearl, J. (2009).
\newblock {\em Causality}.
\newblock Cambridge university press.

\bibitem[Pleiss et~al., 2017]{pleiss2017fairness}
Pleiss, G., Raghavan, M., Wu, F., Kleinberg, J., and Weinberger, K.~Q. (2017).
\newblock On fairness and calibration.
\newblock {\em arXiv preprint arXiv:1709.02012}.

\bibitem[Recht et~al., 2019]{recht2019imagenet}
Recht, B., Roelofs, R., Schmidt, L., and Shankar, V. (2019).
\newblock Do imagenet classifiers generalize to imagenet?
\newblock In {\em International Conference on Machine Learning}, pages
  5389--5400. PMLR.

\bibitem[Sagawa et~al., 2019]{sagawa2019distributionally}
Sagawa, S., Koh, P.~W., Hashimoto, T.~B., and Liang, P. (2019).
\newblock Distributionally robust neural networks.
\newblock In {\em International Conference on Learning Representations}.

\bibitem[Sagawa et~al., 2020]{sagawa2020investigation}
Sagawa, S., Raghunathan, A., Koh, P.~W., and Liang, P. (2020).
\newblock An investigation of why overparameterization exacerbates spurious
  correlations.
\newblock In {\em International Conference on Machine Learning}, pages
  8346--8356. PMLR.

\bibitem[Schmidt et~al., 2018]{schmidt2018adversarially}
Schmidt, L., Santurkar, S., Tsipras, D., Talwar, K., and Madry, A. (2018).
\newblock Adversarially robust generalization requires more data.
\newblock {\em Advances in neural information processing systems}, 31.

\bibitem[Shimodaira, 2000]{shimodaira2000improving}
Shimodaira, H. (2000).
\newblock Improving predictive inference under covariate shift by weighting the
  log-likelihood function.
\newblock {\em Journal of statistical planning and inference}, 90(2):227--244.

\bibitem[Sohoni et~al., 2020]{sohoni2020no}
Sohoni, N.~S., Dunnmon, J.~A., Angus, G., Gu, A., and R{\'e}, C. (2020).
\newblock No subclass left behind: Fine-grained robustness in coarse-grained
  classification problems.
\newblock {\em arXiv preprint arXiv:2011.12945}.

\bibitem[Tsipras et~al., 2018]{tsipras2018robustness}
Tsipras, D., Santurkar, S., Engstrom, L., Turner, A., and Madry, A. (2018).
\newblock Robustness may be at odds with accuracy.
\newblock {\em arXiv preprint arXiv:1805.12152}.

\bibitem[Utama et~al., 2020]{utama2020towards}
Utama, P.~A., Moosavi, N.~S., and Gurevych, I. (2020).
\newblock Towards debiasing nlu models from unknown biases.
\newblock {\em arXiv preprint arXiv:2009.12303}.

\bibitem[Vapnik, 1999]{vapnik1999nature}
Vapnik, V. (1999).
\newblock {\em The nature of statistical learning theory}.
\newblock Springer science \& business media.

\bibitem[Verma et~al., 2019]{verma2019manifold}
Verma, V., Lamb, A., Beckham, C., Najafi, A., Mitliagkas, I., Lopez-Paz, D.,
  and Bengio, Y. (2019).
\newblock Manifold mixup: Better representations by interpolating hidden
  states.
\newblock In {\em International Conference on Machine Learning}, pages
  6438--6447. PMLR.

\bibitem[Woodworth et~al., 2017]{woodworth2017learning}
Woodworth, B., Gunasekar, S., Ohannessian, M.~I., and Srebro, N. (2017).
\newblock Learning non-discriminatory predictors.
\newblock In {\em Conference on Learning Theory}, pages 1920--1953. PMLR.

\bibitem[Yaghoobzadeh et~al., 2019]{yaghoobzadeh2019increasing}
Yaghoobzadeh, Y., Mehri, S., Tachet, R., Hazen, T.~J., and Sordoni, A. (2019).
\newblock Increasing robustness to spurious correlations using forgettable
  examples.
\newblock {\em arXiv preprint arXiv:1911.03861}.

\bibitem[Yao et~al., 2022]{yao2022improving}
Yao, H., Wang, Y., Li, S., Zhang, L., Liang, W., Zou, J., and Finn, C. (2022).
\newblock Improving out-of-distribution robustness via selective augmentation.
\newblock {\em arXiv preprint arXiv:2201.00299}.

\bibitem[Zafar et~al., 2019]{zafar2019fairness}
Zafar, M.~B., Valera, I., Gomez-Rodriguez, M., and Gummadi, K.~P. (2019).
\newblock Fairness constraints: A flexible approach for fair classification.
\newblock {\em The Journal of Machine Learning Research}, 20(1):2737--2778.

\bibitem[Zafar et~al., 2017]{zafar2017fairness}
Zafar, M.~B., Valera, I., Rogriguez, M.~G., and Gummadi, K.~P. (2017).
\newblock Fairness constraints: Mechanisms for fair classification.
\newblock In {\em Artificial Intelligence and Statistics}, pages 962--970.
  PMLR.

\bibitem[Zemel et~al., 2013]{zemel2013learning}
Zemel, R., Wu, Y., Swersky, K., Pitassi, T., and Dwork, C. (2013).
\newblock Learning fair representations.
\newblock In {\em International conference on machine learning}, pages
  325--333. PMLR.

\bibitem[Zhang et~al., 2017]{zhang2017mixup}
Zhang, H., Cisse, M., Dauphin, Y.~N., and Lopez-Paz, D. (2017).
\newblock mixup: Beyond empirical risk minimization.
\newblock {\em arXiv preprint arXiv:1710.09412}.

\bibitem[Zhang et~al., 2020]{zhang2020coping}
Zhang, J., Menon, A., Veit, A., Bhojanapalli, S., Kumar, S., and Sra, S.
  (2020).
\newblock Coping with label shift via distributionally robust optimisation.
\newblock {\em arXiv preprint arXiv:2010.12230}.

\bibitem[Zhang and R{\'e}, 2022]{zhang2022contrastive}
Zhang, M. and R{\'e}, C. (2022).
\newblock Contrastive adapters for foundation model group robustness.
\newblock {\em arXiv preprint arXiv:2207.07180}.

\end{thebibliography}

\clearpage
\onecolumn
\appendix
\section{Preliminary Results without Group Annotation}
While JTT and JM1 do not use full group annotation on the train set, they both rely on group annotation on a small validation set to guide group identification (i.e., choosing early stopping epoch). 
A natural question to investigate is if we can remove the need for group annotation completely, and JM1 can still perform reasonably well.
To that end, we test the robustness of JM1 to different identification epochs/sets in phase I) using the CUB dataset.
Specifically, we checkpoint three different identification scenarios (20, 40, 60 epochs) in phase I) for both JM1 and JTT.
\begin{figure}[ht]
	\centering
	\includegraphics[width=.45\textwidth]{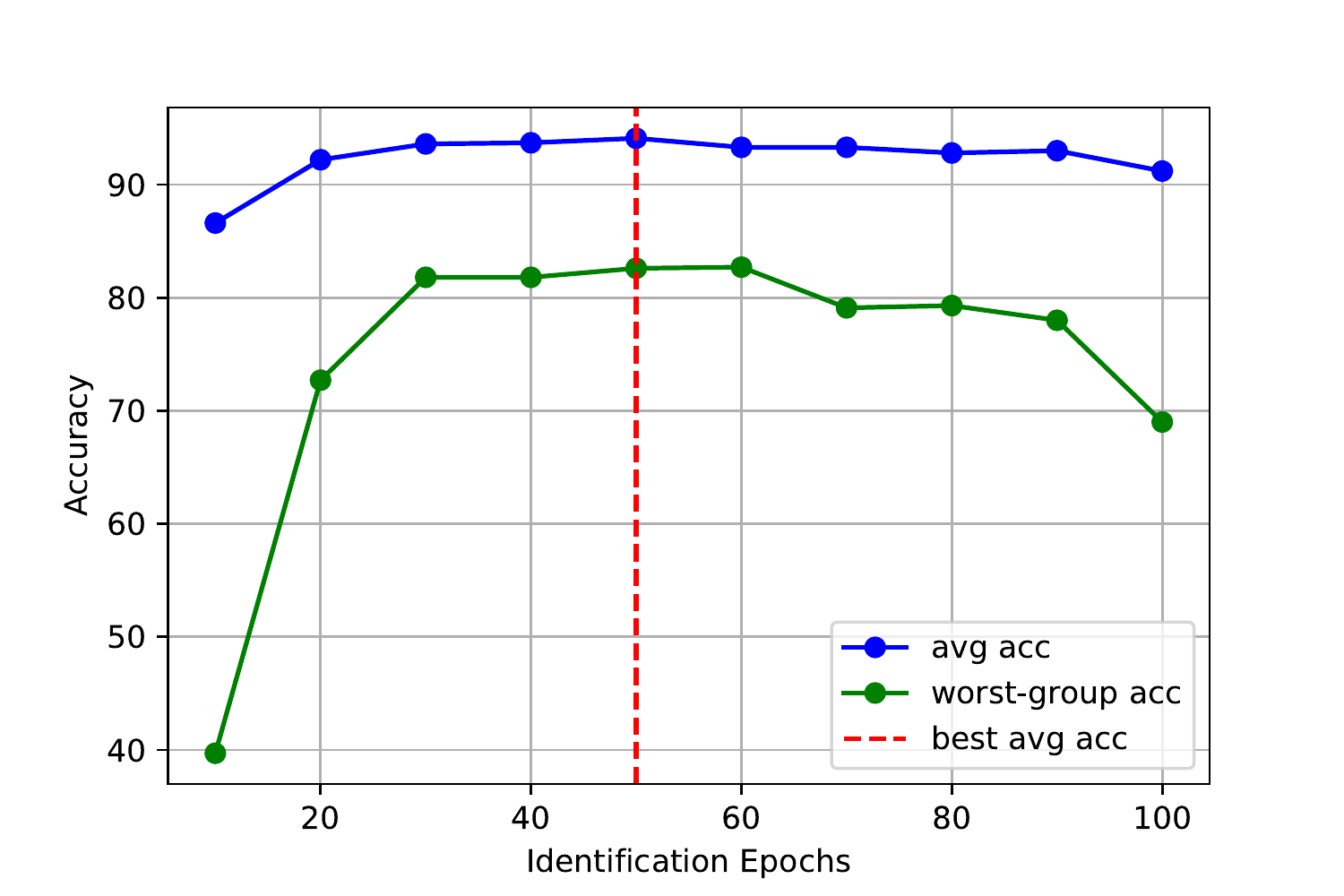}
	\includegraphics[width=.45\textwidth]{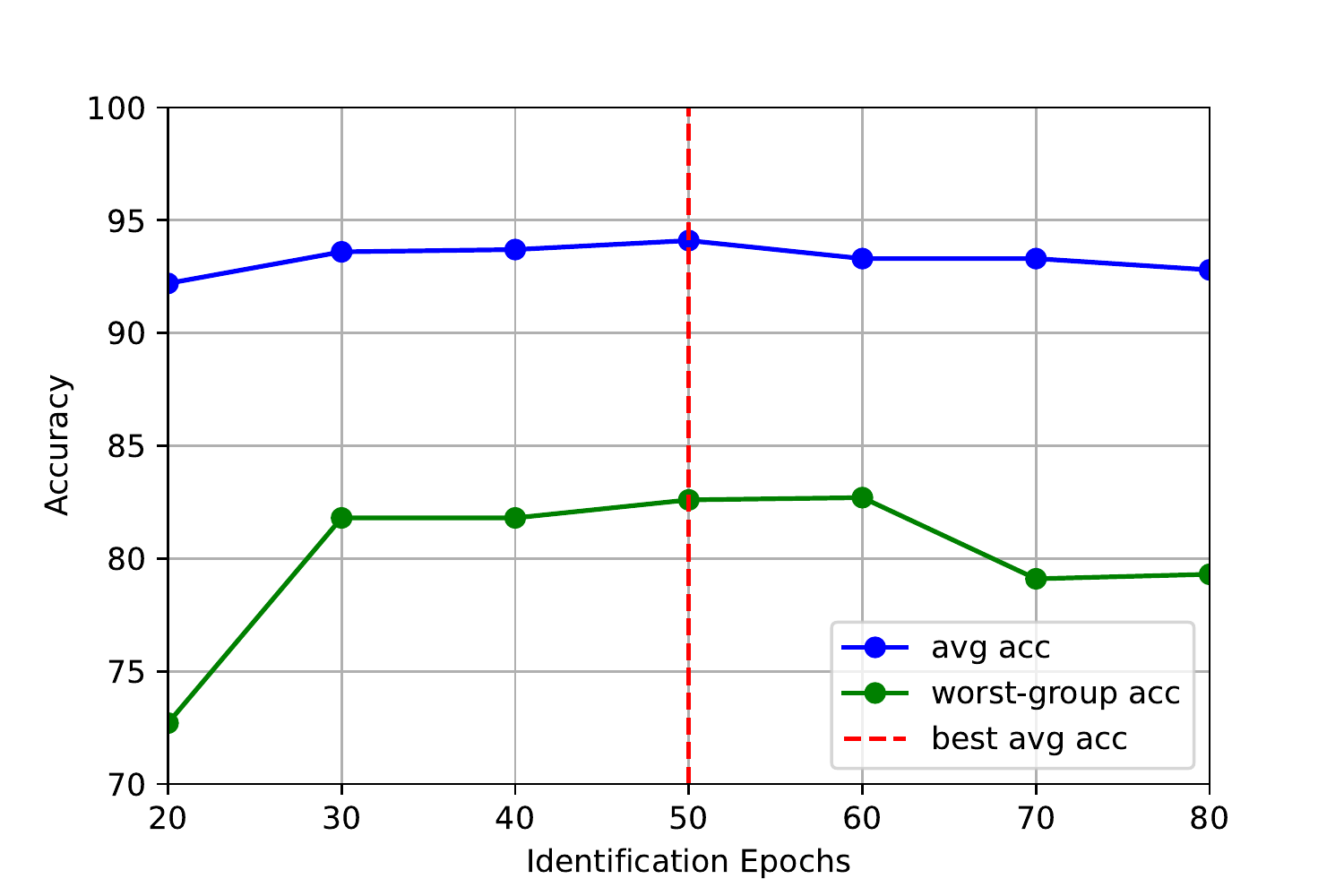}
	\caption{Robustness to identification epoch on CUB dataset without group annotation on the validation
		set. We select the best models using average accuracy for JM1 and worst-group accuracy for JTT.
	}
\end{figure}
Then we select for JM1 the best phase II) model in terms of average accuracy, while we tune the best-performing JTT using group annotation on a validation set in terms of worst-group accuracy.
Note that JM1 is now completely free from group annotations.
Table~\ref{table:cub_no_val} shows that JM1 without group annotation at all loses only around $5\%$ compared to the full JTT with access to group annotation on a validation set.
These are preliminary results that indicate the potential of training group-robust models in a fully self-supervised manner.
We can hope to achieve worst-group robust generalization with no more data requirement than what we need to train typical ML models.
\begin{table}[ht]
	\begin{center}
		\caption{Results on CUB dataset for average accuracy and worst-group accuracy without validation set.
			We checkpoint three identification epochs in the phase I (20, 40, 60 epochs for identification).
			We select JM1 using average accuracy, and tune JTT using worst-group accuracy.
			JM1 handles reasonably well misspecification of the identified error set.
			(S: number of error samples / possibly spurious correlations detected. 
			NS: number of not spurious correlations detected.
			P: precision. R: recall.)}
		\begin{tabular}{l c c c c |  c c | c c c c} 
			\toprule
			& \multicolumn{2}{c}{JM1} & \multicolumn{2}{c}{JTT-best} & \multicolumn{2}{c}{$\Delta \%$} & & & &\\
			\vspace{2pt}
			& avg & worst& avg & worst &  avg & worst & S & NS & P & R\\
			20id             &  92.3  & {72.9} & 78.9& 67.5 & + 17.0& \textcolor{olive}{+ 8.0}
			& 107&290&0.27&0.44\\
			40id             & 93.6   & {83.0} & 90.0& 86.9 & + 4.0 & \textcolor{purple}{- 4.5} 
			& 115&153&0.43&0.48\\
			60id             & 93.2& {82.2} & 90.3& {86.7} & + 3.2 & \textcolor{purple}{- 5.2}
			& 110&128&0.46&0.46\\
			\bottomrule
		\end{tabular}
		\label{table:cub_no_val}
	\end{center}
\end{table}

\begin{figure}[th]
	\centering
	\begin{minipage}{\columnwidth}
		\centering
		\includegraphics[width=0.45\columnwidth]{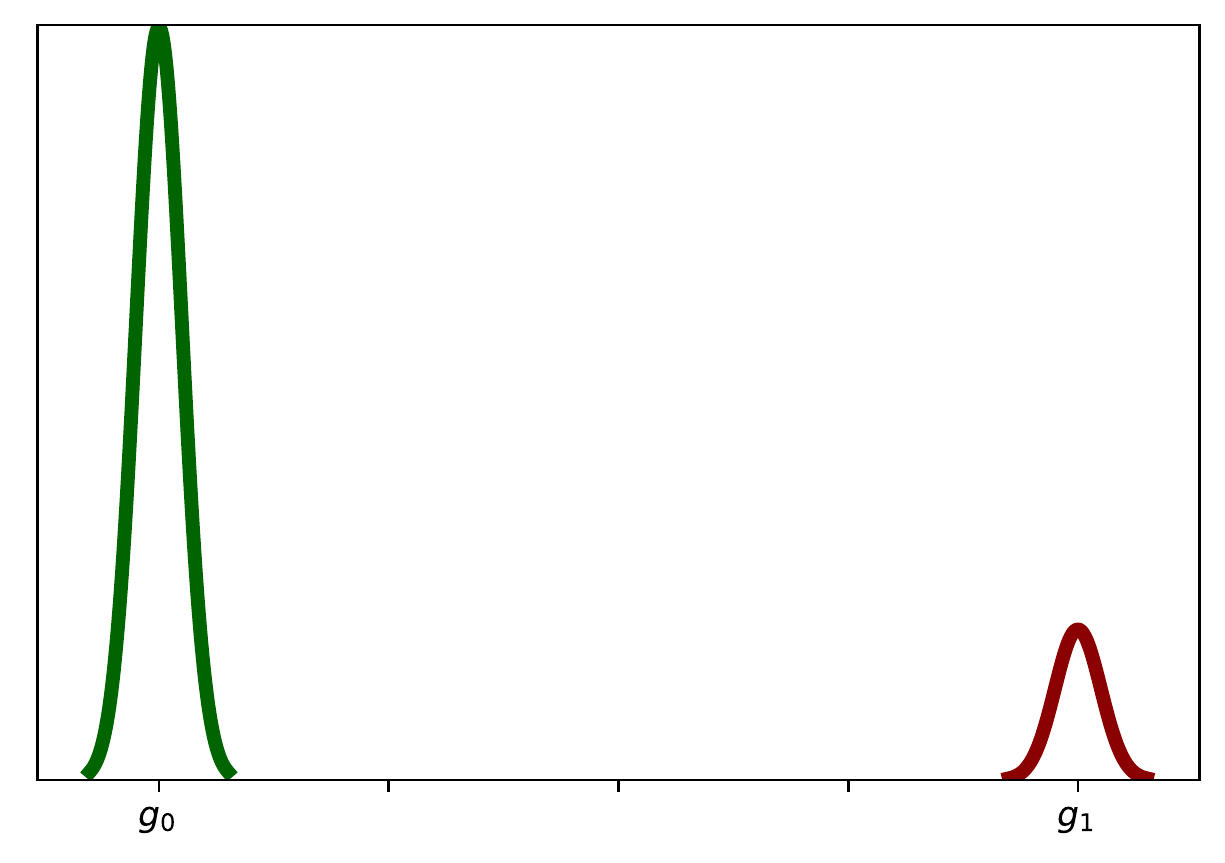}
		\includegraphics[width=0.45\columnwidth]{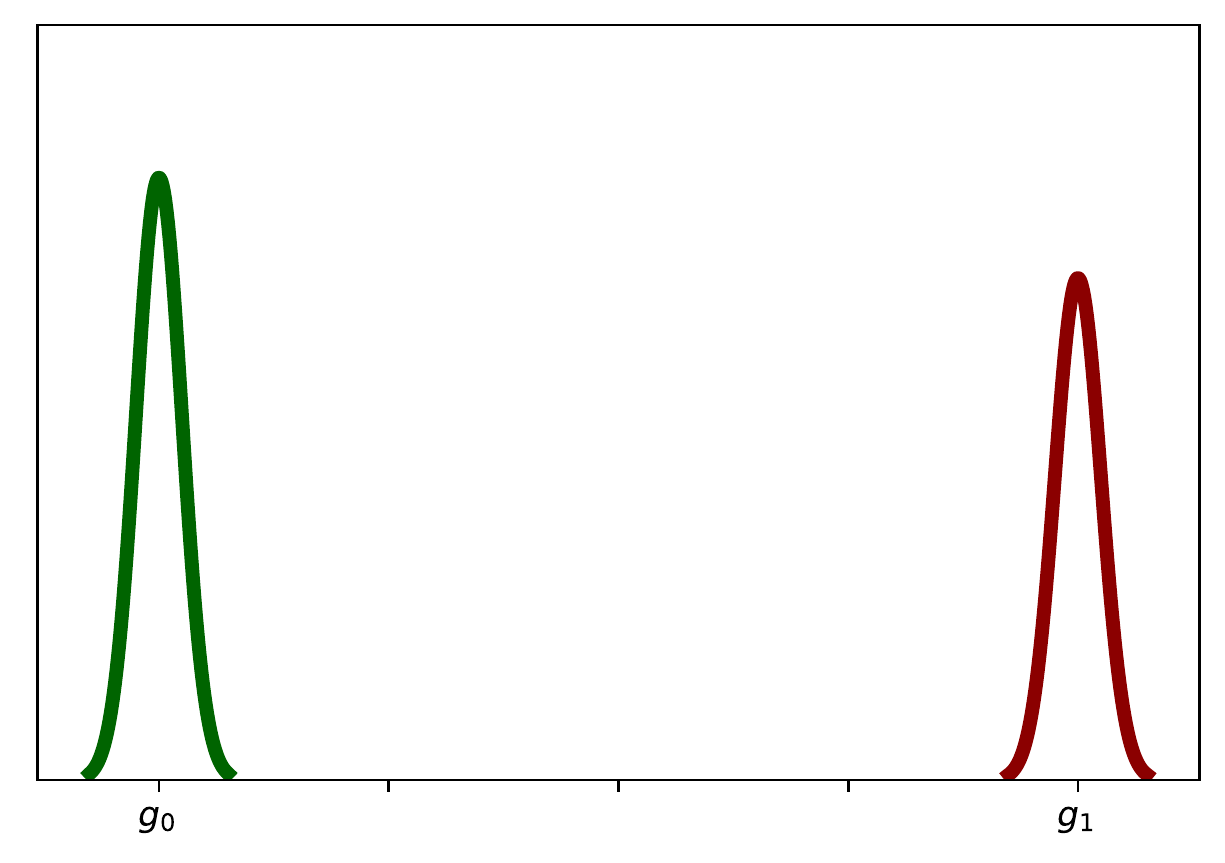}
		\caption{Train and test distro.}
		\label{fig:normal_train_test_distro}
	\end{minipage}%
\end{figure}

\begin{figure}[th]
	\centering
	\begin{minipage}{\columnwidth}
		\centering
		\includegraphics[width=0.45\columnwidth]{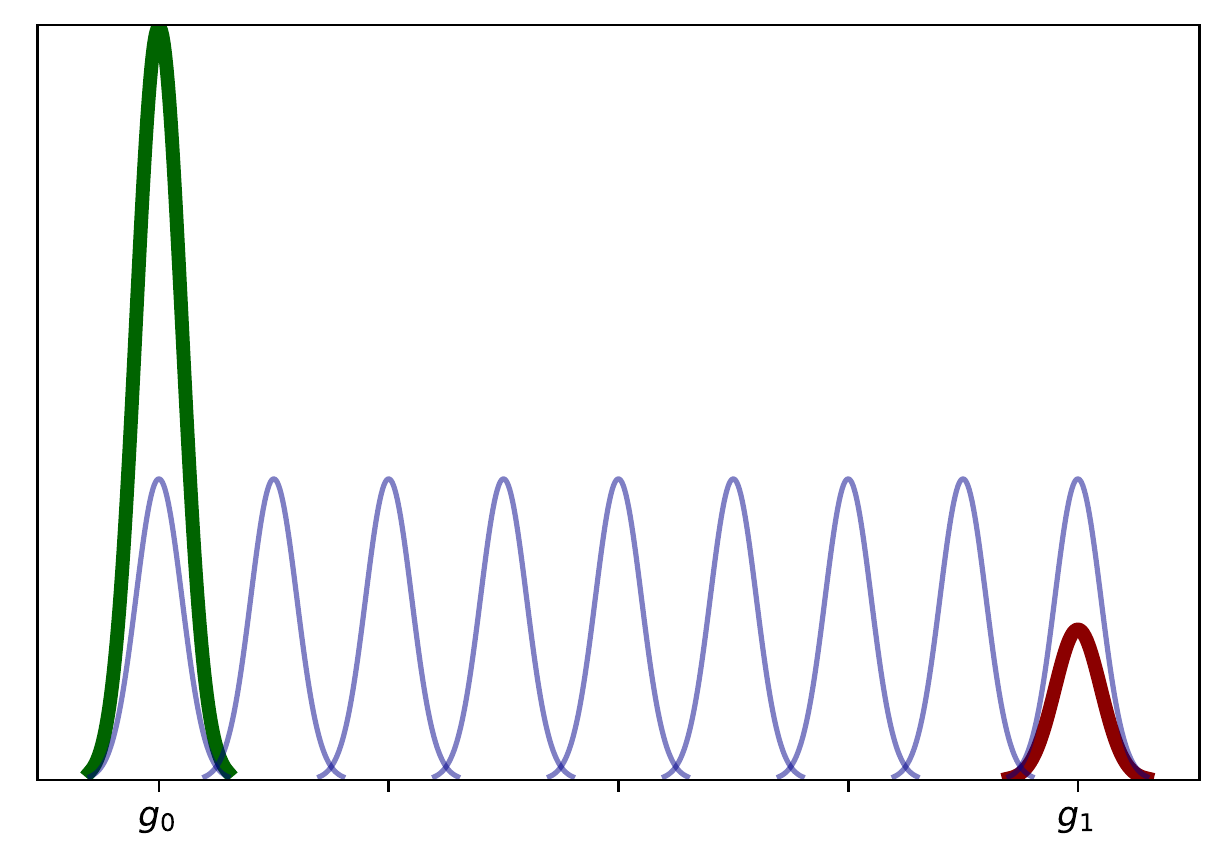}
		\includegraphics[width=0.45\columnwidth]{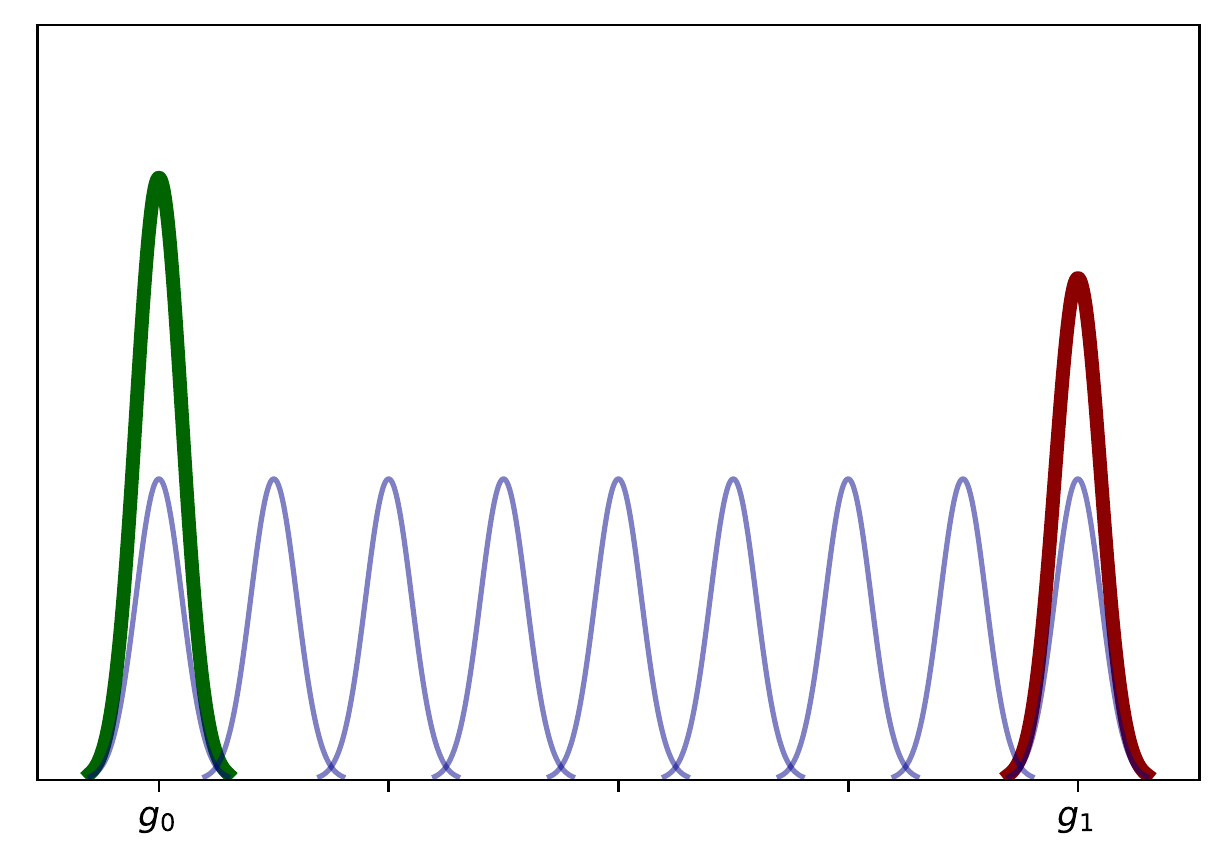}
		\caption{Augmented distro.}
		\label{fig:augmented_train_test_distro}
	\end{minipage}
\end{figure}

\section{Additional Ablation Experiments}
\label{appendix:experiments}
In this section, we report additional ablation experiments.
In Table~\ref{table:layer} we consider the performance of JM1 mixing samples at different levels of abstraction, where the layer indicates mixing in input space, inside the encoder, in output space, or choosing a layer at random. We also test different sampling distributions.
In Table~\ref{table:group-jm1-ablation} we report a similar ablation for the JM1 variant with full group annotation on the train set.
In Figure~\ref{fig:cub-identification-fig} and Table~\ref{table:cub-identification-table} we study the behaviour of the models varying the number of identification epochs in phase I).

	\begin{table}[h]
		\begin{center}
			\caption{Ablation on CUB dataset split by group. The goal is to classify land birds $y_l$ vs water birds $y_w$ in the presence of a confounders: 
			land background $c_l$ vs water background $c_w$.
			JM1+ is a variant of JM1 with additional oversampling.
			U: uniform distribution $U(0, 1)$. 
		    B: beta distribution $B(2, 5)$.}
			\begin{tabular}{l c c c c c c c } 
				\toprule
				\vspace{2pt}
				& layer & $\alpha$ & ($y_l$, $c_l$) & ($y_l$, $c_w$) &  ($y_w$, $c_l$) & ($y_w$, $c_w$) & worst acc \\
				JTT  & - & - & 94.3 & 86.7 & 87.5 & 91.6 & 86.7 \\
				\midrule
				JM1 & input    & $U$         & 98.4 & 84.2 & 77.7  & 93.0  & 77.7 \\
				JM1 & input    & $U(0.5)$ & 96.7 & 84.8 & 86.8  & 93.5  & 84.8 \\
				JM1 & early     & $U$        & 97.5 & 84.7 & 80.1   & 92.2  & 80.1\\
				JM1 & output  & $U$        & 97.3 & 85.9 & 84.7  & 92.7   & 84.7\\
				JM1 &random  & $U$       & 93.8 & 87.8 & 89.7   & 90.8   &87.8\\
				\underline{JM1} & random & $U/B$    & 94.2 & 88.5 &88.9   & 90.5   & \underline{88.5}\\
				\midrule
				JM1+ & input     & $U$   & 91.9    & 89.4   & 90.5   & 84.9 &  84.9\\
				JM1+ & input     & $U/B$ & {96.3} & {87.6} & {88.8} & {93.3} & 87.6\\
				JM1+ &early       & $U$    & 92.5   & 87.9    & 87.5   & 89.9  & 87.5 \\
				\textbf{JM1+} & random & $U/B$  & 94.0  & 89.6    & 90.7   & 90.2 & \textbf{89.6}\\
				\bottomrule
			\end{tabular}
			\label{table:layer}
		\end{center}
	\end{table}

\begin{table}[ht]
	\begin{center}
		\caption{Ablation on CUB dataset for GroupDRO and JM1 variants with different $\alpha$ schedule for the mixing procedure.
		U: uniform distribution $U(0, 1)$. 
		B: beta distribution $B(2, 5)$.}
		\begin{tabular}{l c c c c c} 
			\toprule
			& \multicolumn{2}{c}{CUB} &$\mathcal{L}$ &$\alpha$&mode\\
			\vspace{2pt}
			& avg & \underline{worst}     &&&\\
			GroupDRO   &96.6& 84.6& DRO & - &base\\
			\textbf{GroupDRO}   &93.5& \textbf{91.4}& DRO & - &adj\\
			\midrule
			GroupJM1  &90.3& 89.3& DRO&$B$&adj\\
			GroupJM1  &90.5& 89.9& DRO&$U/B$&adj\\
			
			\textbf{GroupJM1}  &92.2& \textbf{91.3}& DRO&$U$&adj\\
			\midrule
			JM1  &94.2& 86.6& ERM&$U/B$&-\\
			JM1  &94.4& 86.7& ERM&$B$&-\\
			\underline{JM1}  &90.9& \underline{90.5} & ERM&$U$& - \\
			\bottomrule
		\end{tabular}
		\label{table:group-jm1-ablation}
	\end{center}
\end{table}

\begin{figure}[th]
	\begin{minipage}{.5\textwidth}
	\centering
	\includegraphics[width=\textwidth]{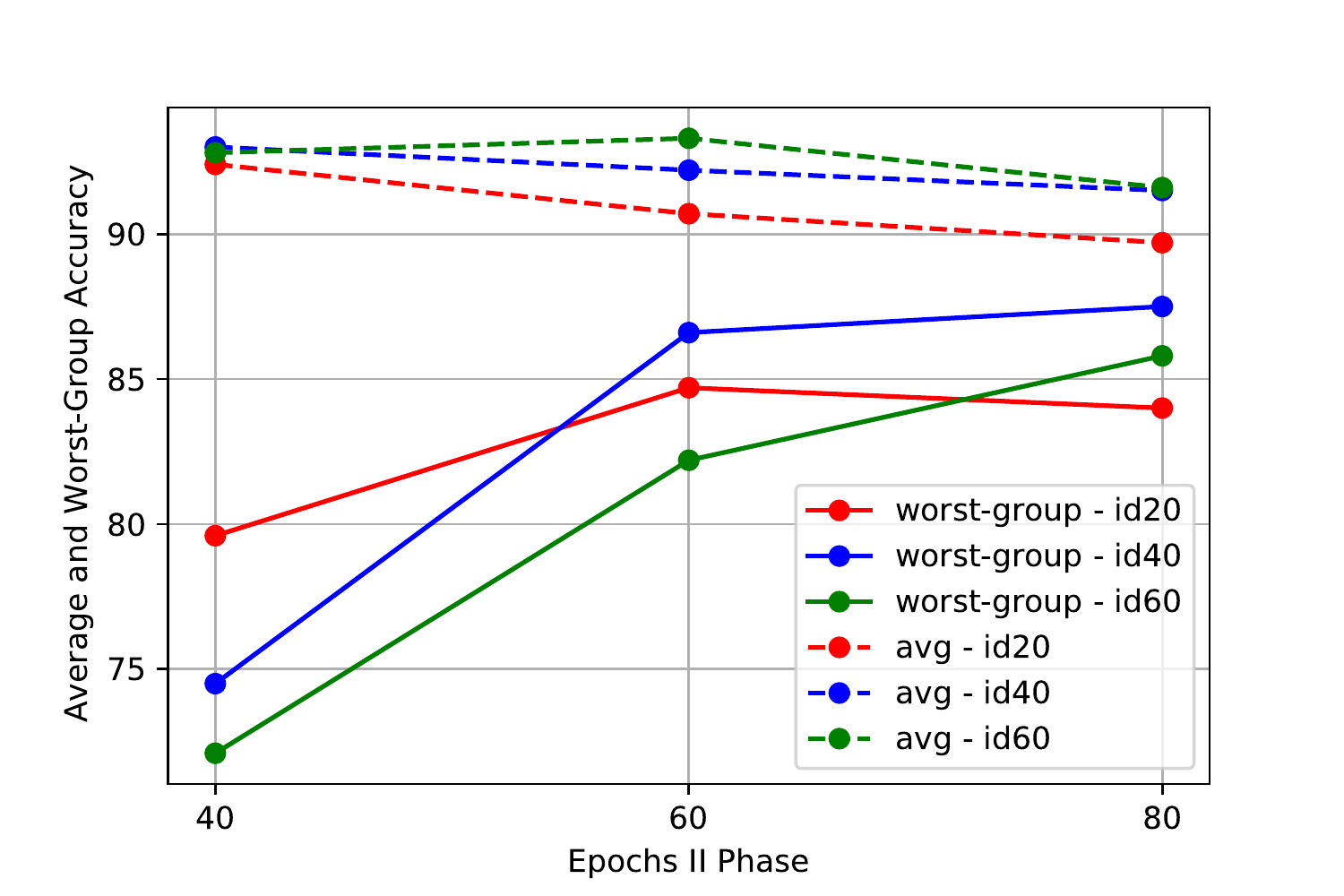}
	\caption{\small Worst-group and average accuracy for different setups.
			On the x-axis we have different epochs we evaluate the model performance in the II phase. 
			Each line represents a different epoch to identify samples in the I phase.
	}
	\label{fig:cub-identification-fig}
	\end{minipage}
	\begin{minipage}{.5\textwidth}
		\begin{center}
			\captionof{table}{\small Results on CUB for different identification epochs. For JM1-avg we choose the model based on average accuracy.}
			\begin{tabular}{l c c | c c} 
				\toprule
				& \multicolumn{2}{c}{JM1-avg} & \multicolumn{2}{c}{JM1-robust}  \\
				& avg & worst & avg & worst\\
				10id &86.6 &39.7 &80.6 &64.2 \\
				20id & 92.2 &72.7 &89.8 &84.4\\
				30id & 93.6 &81.8 &90.8 &86.4\\
				40id &93.7 &81.8 &90.3 &87.5\\
				50id &94.1 &82.6 &90.8 &87.5\\
				60id &93.3 &82.7 &91.1 &87.0\\
				70id &93.3 &79.1 &90.2 &85.8\\
				80id &92.8 &79.3 &90.2 &85.0\\
				90id &93.0 &78.0 &89.9 &84.7\\
				100id &91.2 &69.0 &86.3 &81.3\\
				\bottomrule
			\end{tabular}
			\label{table:cub-identification-table}
		\end{center}
	\end{minipage}
\end{figure}

\begin{table}[!ht]
	\begin{center}
		\caption{Per-group accuracy on MultiNLI. 
			The goal is to classify pair of sentences as entailment $\vy_e$, contradiction $\vy_c$ and neutral $\vy_n$ in the presence of a confounder: negation $\vc_{neg}$ and no negation $\vc_{nneg}$.}
		\begin{tabular}{l c c c c c c c c} 
			\toprule
			& ($\vy_{c}$, $\vc_{nneg}$) & ($\vy_c$, $\vc_{neg}$) &  ($\vy_{e}$, $\vc_{nneg}$) & ($\vy_e$, $\vc_{neg}$) & ($\vy_n$, $\vc_{nneg}$) & ($\vy_n$, $\vc_{neg}$) & avg acc & worst acc \\
			\\
			JTT   & 80.3 & 93.1& 81.0& 73.6& 78.4& 70.6& 80.5& 70.6\\
			\textbf{JM1} & 79.3& 91.1& 81.6& 73.6& 78.5& 73.5& 80.3&\textbf{73.5}\\
			\bottomrule
		\end{tabular}
		\label{table:multinli_appx}
	\end{center}
\end{table}

\clearpage
\section{Experimental Details}
\label{appendix:details}
\begin{table}[ht]
	\begin{center}
		\caption{Training details for JM1.
		CE: cross-entropy. 
		U: uniform distribution $U(0, 1)$. 
		B: beta distribution $Beta(2, 5)$.
		w/o I: without using phase I for cluster identification.
			For all the other hyper-parameters we use the one proposed in~\citep{liu2021just}.
		}
		\begin{tabular}{l c c c  c c} 
			\toprule
			& CUB & CelebA & MultiNLI & CUB (w/o I) & CelebA (w/o I)\\
			&&&&&\\
			Batch Size &32&64&32&32&64\\
			Identification Epoch & 40& 1&2 &- &-\\
			Early stopping &\cmark &\cmark &\cmark  &\cmark&\cmark\\
			Epochs (I Phase)  & 200 &40 &4   & -    & -\\
			Epochs (II Phase) & 200 &40 &4 & 200& 80\\
			$\alpha$ &U/B & U/B & U & U/B& U/B\\
			Learning Rate &1e-5&1e-5&1e-5&1e-5&1e-5\\
			$\mathcal{L}$ &CE&CE&CE&CE&CE\\
			Number Groups &4&4&6 &4&4\\
			Number Classes &2&2&3 &2&2\\
			Mixing Layer &random& early layer& output& random& early layer\\
			Pretrained Encoder &ResNet50&ResNet50&BERT&ResNet50&ResNet50\\
			Weight Decay &1.0&0.1&0.1 &1.0&0.1\\
			\bottomrule
		\end{tabular}
	\end{center}
	\label{table:details}
\end{table}

\end{document}